\documentclass{article}

\PassOptionsToPackage{numbers, compress}{natbib}
\usepackage{natbib}

\usepackage[utf8]{inputenc} %
\usepackage[T1]{fontenc}    %
\usepackage{lmodern}
\usepackage{fullpage}
\usepackage{hyperref}       %
\usepackage{url}            %
\usepackage{booktabs}       %
\usepackage{amsfonts}       %
\usepackage{nicefrac}       %
\usepackage{microtype}      %
\usepackage{tablefootnote}
\usepackage{hyperref}       %
\usepackage{url}            %
\usepackage{amsfonts}       %
\usepackage{nicefrac}       %
\usepackage{amsmath}
\usepackage{color}
\usepackage{standalone}
\usepackage{tikz}
\usepackage{setspace}
\usepackage{amssymb}
\usepackage{bm}
\usepackage{enumitem}

\usetikzlibrary{decorations.pathreplacing}

\usepackage{amsthm}

\usepackage{thm-restate}
\declaretheorem[name=Theorem,refname={Theorem,Theorems},Refname={Theorem,Theorems}]{theorem}
\declaretheorem[name=Lemma,refname={Lemma,Lemmas},Refname={Lemma,Lemmas},sibling=theorem]{lemma}
\declaretheorem[name=Proposition,refname={Proposition,Propositions},Refname={Proposition,Propositions},sibling=theorem]{proposition}
\declaretheorem[name=Corollary,refname={Corollary,Corollaries},Refname={Corollary,Corollaries},sibling=theorem]{corollary}

\newtheorem{assumption}{Assumption}

\DeclareMathOperator*{\argmax}{arg\,max}

\newcommand{\one}{\mathbb{I}}
 
\newcommand{\Reg}{\mathfrak{R}}

\newcommand{\E}{\mathbb{E}}
\newcommand{\EE}[1]{\mathbb{E}[#1]}

\newcommand{\Var}{\mathbb{V}}
\newcommand{\cA}{\mathcal{A}}

\newcommand{\cZ}{\mathcal{Z}}

\newcommand{\cS}{\mathcal{S}}

\newcommand{\cN}{\mathcal{N}}

\newcommand{\cQ}{\mathcal{Q}}

\newcommand{\cP}{\mathcal{P}}

\newcommand{\cC}{\mathcal{C}}
\newcommand{\cL}{\mathcal{L}}

\newcommand{\bR}{\mathbb{R}}

\renewcommand{\>}{\rangle}

\usepackage{array,xspace}
\usepackage[disable]{todonotes}

\newcommand{\eleanor}{\textsc{Eleanor}\xspace}

\newcommand{\val}{V}
\newcommand{\V}{\mathcal{V}}
\newcommand{\M}{\mathcal{M}}

\newcommand{\A}{\mathcal{A}}

\newcommand{\real}{\mathbb{R}}

\newcommand{\Sw}{\mathcal{S}}

\newcommand{\II}[1]{\mathbb{I}_{\left\{#1\right\}}}

\newcommand{\EEpi}[1]{\mathbb{E}_{\bpi}\left[#1\right]}
\newcommand{\PPpi}[1]{\mathbb{P}_{\bpi}\left[#1\right]}

\newcommand{\EEcc}[2]{\mathbb{E}\left[\left.#1\right|#2\right]}

\renewcommand{\th}{\ensuremath{^{\mathrm{th}}}}

\def\argmax{\mathop{\mbox{ arg\,max}}}
\newcommand{\ra}{\rightarrow}
\newcommand{\be}{\mathbf{e}}
\newcommand{\bq}{\mathbf{q}}

\newcommand{\iprod}[2]{\left\langle#1,#2\right\rangle}
\newcommand{\siprod}[2]{\langle#1,#2\rangle}
\newcommand{\biprod}[2]{\bigl\langle#1,#2\bigr\rangle}

\newcommand{\norm}[1]{\left\|#1\right\|}
\newcommand{\bnorm}[1]{\bigl\|#1\bigr\|}

\newcommand{\onenorm}[1]{\norm{#1}_1}
\newcommand{\twonorm}[1]{\norm{#1}_2}
\newcommand{\infnorm}[1]{\norm{#1}_\infty}
\newcommand{\opnorm}[1]{\norm{#1}_{\text{op}}}

\newcommand{\ev}[1]{\left\{#1\right\}}
\newcommand{\bev}[1]{\bigl\{#1\bigr\}}
\newcommand{\pa}[1]{\left(#1\right)}
\newcommand{\bpa}[1]{\bigl(#1\bigr)}

\newcommand{\wh}{\widehat}
\newcommand{\wt}{\widetilde}

\newcommand{\bpi}{\bm{\pi}}

\newcommand{\tV}{\wt{V}}
\newcommand{\tg}{\wt{g}}
\newcommand{\tP}{\wt{P}}
\newcommand{\hP}{\wh{P}}
\newcommand{\hM}{\wh{M}}
\newcommand{\tM}{\wt{M}}
\newcommand{\hp}{\wh{p}}

\newcommand{\CB}{\textup{CB}}

\newcommand{\ttheta}{\wt{\theta}}

\newcommand{\transpose}{^\mathsf{\scriptscriptstyle T}}

\definecolor{PalePurp}{rgb}{0.66,0.57,0.66}

\definecolor{pinkish}{rgb}{0.8,0.2,0.5}
\newcommand{\redd}[1]{\textcolor{pinkish}{#1}}

\newcommand{\todoc}[1]{\textcolor{blue!60!white}{\bf [[Ciara: #1]]}}

\title{A Unifying View of Optimism in Episodic Reinforcement Learning}

\author{%
   Gergely Neu \\
  Universitat Pompeu Fabra\\
  Barcelona, Spain \\
  \texttt{gergely.neu@gmail.com} \\
 \and
  Ciara Pike-Burke\footnote{This work was done while CPB was at Universitat Pompeu Fabra and Barcelona Graduate School of Economics.} \\
 Imperial College London\\
 London, UK \\
   \texttt{c.pikeburke@gmail.com} \\
}

\begin{document}

\maketitle

\begin{abstract}%
The principle of ``optimism in the face of uncertainty'' underpins many theoretically successful reinforcement learning algorithms.
In this paper we provide a general framework for designing, analyzing and implementing such algorithms in the episodic reinforcement learning problem. 
This framework is built upon Lagrangian duality, and demonstrates that every \emph{model-optimistic} algorithm that constructs an optimistic MDP has an equivalent representation as a \emph{value-optimistic} dynamic programming algorithm. 
Typically, it was thought that these two classes of algorithms were distinct, with model-optimistic algorithms 
benefiting from a cleaner probabilistic analysis while value-optimistic algorithms 
are easier to implement and thus more practical.
With the framework developed in this paper, we show that it is possible to get the best of both worlds 
by providing a class of algorithms which have a computationally efficient dynamic-programming implementation and also a simple probabilistic analysis.
Besides being able to capture many existing algorithms in the tabular setting, our framework 
can also address large-scale problems under realizable function approximation, where it enables a simple 
model-based analysis of some recently proposed methods.

\end{abstract}

\section{Introduction}
Reinforcement learning (RL) is a key framework for sequential decision-making under uncertainty \citep{SB18,Sze10}.
In an RL problem, a learning agent interacts with a reactive environment by taking a series of actions. 
Each action provides the agent with some reward, but also takes them to a new state which determines their future rewards. 
The aim of the agent is to pick actions to maximize their total reward in the long run. %
The learning problem is typically modeled by a Markov Decision Process (MDP, \citep{Puterman1994}) where the agent does not know the rewards or transition probabilities. 
Dealing with this lack of knowledge %
 is a crucial challenge in reinforcement learning: %
the agent must maximize their rewards \emph{while simultaneously} learning about the environment. 
One class of algorithms that have been successful at balancing this \emph{exploration versus exploitation} trade-off %
are \emph{optimistic reinforcement learning algorithms}. 
In this paper, we provide a new framework for studying these algorithms.

Optimistic algorithms are built upon the principle of  ``optimism in the face of uncertainty'' (OFU).
They operate by maintaining a set of statistically plausible models of the world, and selecting actions to maximize the returns in the best plausible world. 
Such algorithms were first studied in the context of  %
multi-armed bandit problems \citep{LaiRo85,Agr95,BuKa96,auer2002finite,LSz19book}, and went on to inspire numerous algorithms for reinforcement learning. 
A closer look at the literature reveals two main approaches to incorporate optimism in RL. %
In the first, optimism is introduced through estimates of the MDP: these approaches build a set of plausible MDPs by constructing confidence bounds around the empirical transition and reward functions, and select the policy that generates the highest total expected reward in the best feasible MDP. 
We refer to this family of methods as \emph{model-optimistic}.
Examples of model-optimistic methods include RMAX \cite{BraTen02,KakadeThesis:2003,SziSze10} and UCRL2 \cite{auer2007logarithmic,jaksch2010near,SL08}. 
While conceptually appealing, model-optimistic methods tend to be difficult to implement due to the complexity of jointly optimizing over models and policies. %
Another approach to incorporating optimism into RL is to construct optimistic upper bounds on the \emph{optimal value functions} which are (informally) 
the total expected reward of the optimal policy in the true MDP.
The optimistic policy
greedily picks actions to maximize the optimistic values. 
We refer to this class of methods as \emph{value-optimistic}.
Examples of algorithms in this class are MBIE-EB \cite{SL08}, UCB-VI \cite{azar2017minimax} and UBEV \cite{DLB17}. %
These algorithms compute the optimistic value functions via dynamic programming (cf.~\citealp{Ber07:DPbookVol2}), 
making them computationally efficient and compatible with empirically successful RL algorithms  that are typically based on value functions. %
One downside of these approaches is that their probabilistic analysis is often excessively complex. %

While these two approaches may look very different on the surface, we show in this paper that there is in fact a very strong connection between them. 
Our first contribution is to show that the optimization problems associated with these two problems exhibit strong duality. This implies that that for every model-optimistic approach, there exists an equivalent value-optimistic approach. %
This bridges the gap between the conceptually simple model-optimistic approaches and the computationally efficient value-optimistic approaches. 
This result enables us to develop a general framework for designing, analyzing and implementing optimistic algorithms in the episodic reinforcement learning problem. 
Our framework is broad enough to capture many existing algorithms for tabular MDPs, and for these we provide a simple analysis and computationally efficient implementation. 
The framework can also be extended to incorporate realizable linear function approximation, where it leads to 
a new model-based analysis of two %
value-optimistic algorithms. %
Our analysis involves constructing a new model-optimistic formulation for factored linear MDPs which may be of independent interest.

 \if0
The paper proceeds as follows. In Section~\ref{sec:prob}, we formally define the episodic RL problems.
In Section~\ref{sec:lit}, we discuss related work. In Section~\ref{sec:lps} we present our framework for optimism in tabular MDPs and show how this leads to improved understanding of existing algorithms.
 In Section~\ref{sec:lin}, we present our results for factored linear MDPs.
We conclude in Section~\ref{sec:conc}.
\fi

\vspace{-5pt}
\section{Background on Markov Decision Processes}\label{sec:background}
\vspace{-5pt}
\paragraph{Finite-horizon episodic MDPs.} %
A finite episodic Markov decision process (MDP) is a tuple $(\cS, \cA, H, \alpha, P, r)$ 
where $\cS$ and $\cA$ are the finite sets of states and actions with $S=|\cS|, A=|\cA|$, $H$ is the (fixed) 
episode length and $\alpha$ is the initial state distribution. The transition functions, $P = \{P_h(\cdot|x,a)\}_{h,x,a}$, 
give the probability $P_h(x'| x,a)$ of reaching state $x' \in \cS$ after playing action $a\in \cA$ from state $x \in \cS$ 
at stage $h$ of an episode, and the reward function, $r:\cS \times \cA \to [0,1]$, assigns a 
reward to each state-action pair. 
For simplicity, we assume $r$ is known and deterministic\footnote{The extension to unknown rewards is fairly straightforward using upper confidence bounds on $r$}, and each episode $t$ begins from state $x_{1,t} \sim \alpha$. 
If no further structure is assumed, we call the MDP \emph{tabular}. 
We define a stationary \emph{policy} $\pi: \cS \to \cA$ as a mapping from 
states to actions, and a nonstationary policy as a collection $\bpi = \{\pi_h\}_{h=1}^H$ 
of stationary policies for each stage $h$ of an episode,
and note that these 
are sufficient %
for maximizing reward in an episode. %
We denote by $\PPpi{\cdot}$ and $\EEpi{\cdot}$ a probability or expectation 
with respect to the distribution of state-action sequences under policy $\bpi$ in the MDP, and let $[H]=\{1,\dots, H\}$ and $\mathcal{Z}=\Sw\times\A$.

\vspace{-5pt}
\paragraph{Value functions and dynamic programming.} 
For any policy $\bpi$, we define the \emph{value function} at each state $x \in \cS$ and stage $h$ as the expected total reward  from running policy $\bpi$ from that point on:
\[
V_h^{\pi}(x) = \E_{\bpi}\bigg[ \sum_{l=h}^H r_l(x_l,\pi_l(x_l)) \bigg| x_h = x \bigg].
\]
We denote by $\bpi^*$ an \emph{optimal policy} satisfying $V_h^{\pi^*}(x) = \max_{\pi} V_h^{\pi}(x)$ for all $x\in \cS,h \in [H]$, and the %
 \emph{optimal value function} by $V_h^*(x) = V_h^{\pi^*}(x)$.  
The total expected reward of $\bpi^*$ in an episode starting from state $x_1$ is $V_1^*(x_1)$.
We define the \emph{(optimal) action-value function} for each $x,a,h$ as
\[
 Q_h^{\pi}(x,a) = \E_{\bpi}\bigg[ \sum_{l=h}^H r_l(x_l,\pi_l(x_l)) \bigg| x_h = x , a_h = a\bigg] 
 \qquad\mbox{and}\qquad Q_h^{*}(x,a) = \max_{\pi} Q_h^{\pi}(x,a).
\]
It is easily shown that the value functions satisfy the \emph{Bellman equations} for all $x,a,h$:
\begin{align*}
\begin{aligned}
&V _h^{\pi}(x) = Q_h^{\pi}(x,\pi(x)) , \quad V_{H+1}^{\pi}(x) = 0
\\ &Q_h^{\pi}(x,a) = r_h(x,a) + \sum_{y \in \cS} P_h(y|x,a)V_{h+1}^{\pi}(y)
\end{aligned}
\, \text{ and } \,
\begin{aligned}
&V _h^*(x) = \max_{a \in \cA} Q_h^*(x,a), \quad V_{H+1}^*(x) = 0
\\ &Q_h^*(x,a) = r_h(x,a) + \sum_{y \in \cS} P_h(y|x,a)V_{h+1}^*(y).
\end{aligned}
\end{align*}
In a fixed MDP, an optimal policy can be found by solving the above system of equations by backward recursion through the stages $H,H-1,\dots,1$, a method known as \emph{dynamic programming} \citep{bellman57,howard60,Ber07:DPbookVol2}.
 
\paragraph{Optimal control in MDPs by linear programming.}
A key technical tool underlying our results is a classic linear-programming (LP) formulation for solving MDPs \citep{Man60,Ghe60,Den70}.
To state this formulation, we will represent value functions by $S$-dimensional vectors and
define the $S\times S$ \emph{transition matrix} $P_{h,a}$ for each $h,a$, acting on a value function $V$ as 
$\pa{P_{h,a} V}(x) = \sum_{x'}P_{h,a}(x'|x,a) V(x')$. Then, the following LP can be seen to be equivalent to the Bellman optimality equations:
\begin{equation}\label{eq:V_LP}
 \underset{V}{\text{minimize}} \quad V_1(x_1) 
 \left| 
 \hspace{-10pt}
 \begin{array}{lll}
&\text{subject to} 
\\& V_h \geq r_a + P_{h,a} V_{h+1} & \forall a \in \cA, h \in [H],
 \end{array} \right.
\end{equation}
where the inequality is to be understood to hold entrywise.
Defining the vector $q_{h,a} = (q_h(x_1,a), \dots, q_h(x_S,a))\transpose$, the dual of the above LP is given as
\begin{equation}\label{eq:primal_full}
\underset{q \in \cQ(x_1)}{\text{maximize}} \iprod{q_{h,a}}{r_a}
\left|
\begin{array}{lll}
&\text{subject to} 
\\ & \sum_a q_{h+1,a} = \sum_a P_{a,h}\transpose q_{h,a} \qquad &\forall x \in \cS, h\in [H], %
\end{array} \right.
\end{equation}
for $\cQ(x_1) = \{ q\in\real_+^{\Sw\times\A \times H}: \sum_a q_{1}(x,a) = \one \{x=x_1\},  q_{h}(x,a)\ge 0 \, \forall (x,a) \in \cZ, h \in [H] \}$.
Feasible points of the above LP can be interpreted as \emph{occupancy measures}. %
For a fixed policy $\bpi$, the occupancy measure $q^{\bpi}$ of policy $\bpi$ at the state-action pair $x,a$ is defined as
$q^{\bpi}_h(x,a) = \PPpi{x_h =x, a_h=a}$.
It can be shown that the set of occupancy measures is uniquely characterized by  $\cQ(x_1)$ and the constraint in~\eqref{eq:primal_full}. 
Each feasible $q$ induces a stochastic policy $\pi^q$ defined as $\pi^q_h(a|x) = \frac{ q_h(x,a)}{\sum_{a' \in \cA}  q_h(x,a')}$ if the denominator is nonzero, and defined arbitrarily otherwise. The optimal solution $q^*$ to the LP in \eqref{eq:primal_full} can be shown to induce an optimal policy $\bpi^{*}$ which satisfies the Bellman optimality equations.
For proofs and further details of this formulation, see \citet{Puterman1994}.

\paragraph{Linear function approximation in MDPs.}
In most practical problems, %
the state space is too large to use the above results and it is common to  
work with parameterized estimates of the quantities of interest.
We focus on the classic idea of 
\emph{linear function approximation} to represent the action-value functions as linear functions of some fixed $d$-dimensional 
feature map $\varphi:\Sw\ra\real^d$, so $Q_h^\theta(x,a) = \iprod{\theta_{h,a}}{\varphi(x)}$ for some $\theta_{h,a}\in\real^d$ for each action $a$ and stage $h$. 
To avoid technicalities, we assume that the state space $\Sw$ is still finite, although potentially very large.
This allows us to define the $\Sw \times d$ \emph{feature matrix} $\Phi$ with its $x$\th row being $\varphi\transpose(x)$, and
represent the action-value function as $Q_{h,a} = \Phi\theta_{h,a}$.
We make the following assumption:
\begin{assumption}[Factored linear MDP \citep{factlin,PS16,jin2020provably}]\label{assn:real}
For each action $a$ and stage $h$, there exists a $d\times\Sw$ matrix $M_{h,a}$ and a vector $\rho_a$ such that the transition matrix can be written 
as $P_{h,a} = \Phi M_{h,a}$, and the reward function as $r_a = \Phi \rho_a$. Furthermore, the rows of $M_{h,a}$, $m_{h,a}(x)$, satisfy 
$\onenorm{m_{h,a}(x)}\le C_P$ for all $(x,a,h)$, $\rho$ satisfies $\twonorm{\rho_a} \le C_r$, and %
$\twonorm{\varphi(x)}\le R$ for some positive constants $C_P, C_r, R$.
\end{assumption}
As shown by \citet{jin2020provably}, this assumption implies that %
for every policy $\bpi$,
there exists a $\theta^{\bpi}$ such that $Q^{\bpi}_h(x,a) = \siprod{\theta^{\bpi}_{h,a}}{\varphi(x)}$. 
We now show that factored linear MDPs also enjoy a strong dual realizability property. %
Let $W_{h,a}$ be an arbitrary symmetric  $\Sw\times\Sw$ \emph{weight matrix}
for each action $a$ such that $\Phi\transpose W_{h,a}\Phi$ is full rank, and notice that, due to the 
realizability of the action-value functions,
the optimal value functions 
 can be written as
 the solution to the following LP:
\begin{equation*}
 \underset{V,\theta}{\text{minimize}} \quad V_1(x_1)
 \left| 
 \begin{array}{rll}
 \text{subject to}\\
 \quad\theta_{h,a} &= \pa{\Phi\transpose W_a\Phi}^{-1} \Phi\transpose W_{h,a} \pa{r_a + P_{h,a} V_{h+1}} &\forall  h\in[H],
 \\
 \quad V_{h} &\geq \Phi\theta_{h,a} &\forall a \in \cA, h \in [H].
 \end{array}
 \right.
\end{equation*}
Under Assumption~\ref{assn:real}, this LP is feasible and has a finite solution. 
It also holds that
parameter vectors $\theta_{h,a}$ are independent of the choice of the weight matrix $W_{h,a}$.
The dual of this LP can be written as
\begin{equation}\label{eq:linprimal}
\underset{q \in \cQ(x_1), \omega}{\text{maximize}} \quad  \sum_{h=1}^H \sum_a\iprod{\Phi \omega_{h,a}}{r_a} 
 \left| 
 \hspace{-10pt}
 \begin{array}{rll}
 \text{subject to}
 \\
 \quad \sum_a {\bf q}_{h+1,a} &= \sum_a P_{h,a}\transpose W_{h,a} \Phi \omega_{h,a} \qquad &\forall h\in[H]
 \\
 \quad \Phi\transpose {\bf q}_{h,a} &= \Phi\transpose W_{h,a} \Phi \omega_{h,a} \qquad &\forall a \in \cA, h\in[H] 
 \end{array}\right.
\end{equation}
Due to the boundedness and feasibility of the primal LP, the dual is also feasible and bounded. 
Moreover, any vector ${\bf q}$ that is feasible for \eqref{eq:linprimal} 
is also a feasible solution to the full 
LP~\eqref{eq:primal_full}, since %
\[
\sum_a {\bf q}_{h+1,a} = \sum_a P_{h,a}\transpose \Phi W_{h,a} \omega_{h,a} =  \sum_a M_{h,a}\Phi\transpose W_{h,a} \Phi \omega_{a,h} = 
\sum_a M_{h,a}\Phi\transpose \bq_{h,a} = \sum_a P_{h,a} \bq_{h,a}.
\]
Thus, for factored linear MDPs, the set of occupancy measures is \emph{exactly} characterized by the constraints in~\eqref{eq:linprimal}. 
To the best of our knowledge, these LP formulations and results are novel and may have other uses beyond the setting of factored linear MDPs. 
For instance, MDPs exhibiting zero inherent Bellman error \citep{zanette2020learning} can be also seen to 
yield a feasible and finite solution for both LPs, although the above dual realizability property is not guaranteed to hold for all occupancy measures.

\section{Regret Minimization in  Episodic Reinforcement Learning} \label{sec:lit}
We consider algorithms that sequentially interact with a fixed but \emph{unknown} MDP over $K$ episodes.
In each episode, $t$, the algorithm selects a policy $\bpi_t$ with the aim of maximizing the cumulative reward in that episode.
We assume that the learner has no prior knowledge of the transition function, and can only learn about the MDP through interaction. 
 The performance is measured in terms of the \emph{regret}, %
\[
\Reg_T = \sum_{t=1}^K ( V^*_1(x_{1,t}) - V_1^{\bpi_k}(x_{1,t}))
\]
where $T=KH$ is the total number of rounds and $x_{1,t} \sim \alpha$ is the initial state in episode $t$. 
 
 In tabular MDPs, the lower bound on the regret is $\Omega(H\sqrt{SAT})$ \citep{jaksch2010near,osband2016lower,jin2018q}\footnote{The extra $\sqrt{H}$ due to having a different $P_h$ per stage. We use $\wt{O}(\cdot)$ to denote order up to logarithmic terms.}.%
 Most optimistic algorithms are either model-optimistic or value-optimistic.
 Some notable model-optimistic approaches %
 are UCRL2 \citep{jaksch2010near} and REGAL \citep{bartlett2012regal} which have regret $\wt{O}(S\sqrt{H^3AT})$, %
and KL-UCRL \citep{filippi2010optimism,talebi2018variance} %
and UCRL2-B \citep{improved_analysis_UCRL2B}, which have regret $\wt{O}(H\sqrt{S\Gamma AT})$ where $\Gamma \leq S$ is the maximal number of reachable states from any $(x,a) \in \cZ$ and stage $h \in [H]$. %
 These algorithms differ predominantly in the choice of distance and concentration bounds defining the set of feasible transition functions.
Value-optimistic approaches often enjoy low regret at a cost of a more complex analysis. Some examples of these include UBEV \citep{DLB17} which has regret $\wt{O}(\sqrt{H^5SAT})$, and UCB-VI \citep{azar2017minimax} which has regret $\wt{O}(H\sqrt{SAT})$, matching the lower bound. %
We note that %
optimism has also been used in the model free setting (e.g. \cite{jin2018q}), and that other non-optimistic approaches have also been successful at regret minimization (see e.g. \cite{osband2013more,agrawal2017optimistic}).
Other related works include
\cite{zimin2013online,rosenberg2019online} which also use occupancy measures,  \cite{tewari2008optimistic} where optimistic linear programs are used, and \cite{maillard2014hard,talebi2018variance} which exploit duality in specific cases.

For factored linear MDPs, all optimistic algorithms we are aware of
 are value-based, without a clear model-based interpretation: LSVI-UCB \cite{jin2020provably} uses dynamic programming and has regret $\wt{O}(\sqrt{d^3H^3T})$, while \textsc{Eleanor} \cite{zanette2020learning} has regret $\wt{O}(Hd\sqrt{T})$ but requires solving a complex optimization problem in each episode. 
The UC-MatrixRL algorithm \cite{yang2019reinforcement} considers a %
 different problem %
 with two feature maps but is model-based with regret $\wt{O}(H^2d\sqrt{T})$. %
Non-optimistic approaches %
 include \cite{russo2019worst,zanette2019frequentist}.

\section{Optimism in Tabular Reinforcement Learning} \label{sec:tab}
 We now present our main contribution: a general framework for designing, analyzing and implementing 
 optimistic RL algorithms in episodic tabular MDPs. Our framework naturally extends the LPs in~\eqref{eq:V_LP} %
 and~\eqref{eq:primal_full} %
  to account for \emph{uncertainty} %
 about the transition function. 
 We use confidence intervals for the transition functions to express uncertainty in the space of occupancy measures %
 and maximize the expected reward over this  set.
 Our key result shows that the dual of this optimization problem can be written in dynamic-programming form with added
 exploration bonuses, the size of which are
determined by the shape of the primal confidence sets. %
 
We define the uncertainty sets using confidence intervals around a reference transition function $\wh{P}$. 
For a divergence measure $D(p,p')$ %
between probability distributions $p,p'$, define the confidence sets %
\begin{align}%
\cP = \ev{ \wt{P} \in \Delta: D\pa{\tP_h(\cdot|x,a) , \wh{P}_h(\cdot|x,a)} \leq \epsilon(x,a) \quad \forall (x,a) \in \cS\times\cA, h \in [H]},
\label{eq:Pdef}
\end{align}
where $\Delta$ is the set of valid transition functions. %
 We assume that the divergence measure $D$ is jointly convex in its arguments so that $\cP$ is convex, and that $D$ is positive homogeneous
so for any $\alpha\ge 0$, $D(\alpha p,\alpha p') = \alpha D(p,p')$.
Note that the distance 
$\norm{p - p'}$ for any norm and all $f$-divergences satisfy these conditions \citep{LV06}. %
Using %
$\cP$, we modify %
\eqref{eq:primal_full} %
to get the optimistic primal optimization problem, %
\begin{equation}\label{eq:primalMDPopt}
\underset{\substack{q \in \cQ(x_1) \\ \wt P \in \Delta}}{\text{maximize}}  \quad \sum_{h=1}^H \iprod{q_{h,a}}{r} 
 \left| 
 \hspace{-10pt}
 \begin{array}{lll}
& \text{subject to} & 
 \\ 
 &\sum_{a} q_{h+1,a} = \sum_a \wt{P}_{h,a}\transpose q_{h,a} \,  &\forall \,h\in [H] 
 \\
 & D\pa{\wt P_h(\cdot|x,a), \wh{P}_h(\cdot|x,a)} \leq \epsilon(x,a) \, &\forall (x,a) \in \cZ, \, h \in [H]
 \end{array}\right.
\end{equation}

We pick $\epsilon$ such that $P \in \cP$ with high probability. In this case, the above optimization problem 
returns an ``optimistic'' occupancy measure with higher expected reward %
than the true optimal policy.
Unfortunately, the optimization problem in \eqref{eq:primalMDPopt} is not convex due to the bilinear constraint $q_{h+1,a} = \sum_a \wt{P}_{h,a}\transpose q_{h,a}$. 
Our main result below shows that it is %
still possible to obtain an equivalent %
value-optimistic formulation via Lagrangian duality and an appropriate reparametrization. %
We make use of the \emph{conjugate}
of the divergence $D$ defined for any function $z$, distribution $p'$ and threshold $\epsilon$ as
\[
 D_*\pa{z\middle| \epsilon, p'} = \max_{p \in \Delta} \ev{\iprod{z}{p-p'} \middle| D(p,p') \leq \epsilon}.
\]
\vspace{-5pt}
\begin{restatable}{proposition}{dual}\label{prop:dual}
Let $ \CB_h(x,a) = D_*(\val_{h+1}| \epsilon_h(x,a),\wh{P}_h(\cdot|x,a))$ and denote its vector representation by $\CB_{h,a}$. 
The optimization problem in \eqref{eq:primalMDPopt} can be equivalently written as
\begin{equation} \label{eq:dualtab}
\underset{\val}{\text{minimize }}   \val_1(x_1)
\left| 
 \begin{array}{lll}
&\text{subject to } 
\\& \val_h \geq r_a + \hP_{h,a} \val_{h+1} + \CB_{h,a} & \forall a \in \cA, h \in [H]
\end{array} \right.
\end{equation}
\end{restatable}
\vspace{-10pt}
\begin{proof}[Proof sketch]
The full proof is in Appendix~\ref{app:dual}. Here we outline the key ideas.
To show strong duality, we reparameterize the problem as follows: 
define $J_h(x,a,x')=\wt{P}_h(x'|x,a)q_h(x,a)$
and note that 
due to homogeneity of $D$, the %
constraint on $\wt{P}$ is equivalent to $D(J_h(x,a,\cdot), \wh{P}_h(\cdot|x,a)q_h(x,a)) \leq \epsilon_h(x,a) q_h(x,a)$, which is convex in 
 $q$ and $J$. %
It is straightforward to verify the Slater condition for the resulting convex program, and thus strong duality holds for both parametrizations.

Letting $\cL(q, \wt{P};\val)$ be the Lagrangian of \eqref{eq:primalMDPopt} and using the non-negativity of $q$, 
the maximum of \eqref{eq:primalMDPopt} is 
\begin{align}
\min_\val \max_{\substack{q\geq 0 \\ \wt P \in \cP}} \cL(q, \wt P;\val)=& \min_\val \max_{q \geq 0} \bigg \{ \sum_{x,a,h} q_h(x,a) \bigg( \sum_y \wh P_h(y|x,a) \val_{h+1}(y) +r(x,a) - \val_h(x) \nonumber
\\ &   + \max_{\wt P_h(\cdot|x,a) \in \cP_h(x,a)} \sum_y \pa{ \tP_h(y|x,a) - \hP_h(y|x,a)} \val_{h+1}(y)  \bigg) \bigg \}.  \label{eq:Lmax}
\end{align}
Then, letting $\hp = \wh{P}_h(\cdot|x,a)), \wt{p}(x') = \wt{P}_h(\cdot|x,a)$, and using the definition of $D$ and $D_*$, the inner maximum can be written as
$\max_{\wt{p} \in \Delta} \{\iprod{\val_{h+1}}{\wt{p} - \hp} ; D(\wt{p},\hp) \leq \epsilon_h(x,a)\}  = D_*(\val_{h+1} |  \epsilon_h(x,a), \hp ).$
We then substitute this into \eqref{eq:Lmax} and use standard techniques to get the dual from the Lagrangian.%
\end{proof}
\vspace{-5pt}This result enables us to establish a number of important properties 
of the optimal solutions of the optimistic optimization problem~\eqref{eq:primalMDPopt}.
The following two propositions (proved in in Appendix~\ref{app:properties}) highlight that optimal 
solutions to ~\eqref{eq:primalMDPopt} are optimistic, bounded, and can be found by a dynamic-programming procedure.
This implies that any model-optimistic algorithm that solves \eqref{eq:primalMDPopt} in each episode 
is equivalent to value-optimistic algorithm using an appropriate choice of \emph{exploration bonuses}.
\vspace{-5pt}
\begin{restatable}{proposition}{pols}%
 \label{prop:pols}
Let $\val^+$ be the optimal solution to \eqref{eq:dualtab} and $\textup{CB}^{+}_{h}(x,a)=  D_*(\val_{h+1}^+|\epsilon(x,a),\wh{P}_h)$. Then, the optimal policy $\bpi^+$ extracted from any optimal solution $q^+$ of the primal LP in \eqref{eq:primalMDPopt} satisfies 
\begin{equation}\label{eq:optpolicy}
 \val^+_{h}(x) = r(x,\pi_h^+(x))+ \textup{CB}_{h}^{+}(x,\pi_h^+(x))  + \sum_{y \in \cS} \hP_h(y|x,\pi_h^+(x)) \val_{h+1}^+(y)  \quad \forall x\in\cS, h\in[H].
\end{equation}
\end{restatable}
\vspace{-5pt}
\begin{restatable}{proposition}{valbounds} 
\label{prop:valbounds}
If the true transition function $P$ satisfies the constraint in Equation~\eqref{eq:primalMDPopt}, the optimal solution $\val^+$ of the dual LP satisfies $V^*_h(x) \le \val^+_h(x) \le H-h+1$ for all $x\in\cS$.
\end{restatable}

\if0
The constraints on $\val_h$ in \eqref{eq:dualtab} depend only on the values in stages $h+1,\dots, H$, so the optimal solution can be found efficiently via dynamic programming. 
By Propositions~\ref{prop:pols}~and~\ref{prop:valbounds}, %
this solution %
is an optimal solution to the primal in \eqref{eq:primalMDPopt} and is optimistic.
This implies that any model-optimistic algorithm that solves \eqref{eq:primalMDPopt} in each episode %
is equivalent to an optimistic dynamic programming algorithm.
\fi

\subsection{Regret bounds for optimistic algorithms}
We consider algorithms that, in each episode $t$, define the confidence sets $\cP_t$ in \eqref{eq:Pdef} using some divergence measure $D$ and the reference model $\hP_{h,t}(x'|x,a) = \frac{N_{h,t}(x,a,x')}{N_{h,t}(x,a)}$ $\forall x,x' \in \cS, a \in \cA$.
Here $N_{h,t}(x,a,x')$ is the total number of times that we have played action $a$ from state $x$ in stage $h$ and landed in state $x'$ up to the beginning of episode $t$, and $N_{h,t}(x,a) = \max\{\sum_{x'} N_{h,t}(x,a,x'),1\}$. 
In episode $t$, the algorithm follows the optimistic policy $\bpi_t$ extracted from the solution of the primal optimistic 
 problem in~\eqref{eq:primalMDPopt}, or equivalently, the optimistic dynamic programming procedure in \eqref{eq:dualtab}.
The following theorem establishes a regret guarantee of the resulting algorithm:
\begin{restatable}{theorem}{reggen} 
\label{thm:reggen}
On the event $\cap_{t=1}^K \{P \in \cP_t\}$, the regret is bounded with probability at least $1-\delta$ as
\[
\Reg_T \leq  \sum_{t=1}^K \sum_{h=1}^H \bigg(\textup{CB}_{h,t}(x_{h,t},\pi_{h,t}(x_{h,t})) + \CB_{h,t}^-(x_{h,t},\pi_t(x_{h,t})) \bigg)+H\sqrt{2T\log(1/\delta)} 
\]
where $\CB_{h,t}^-(x,a)= D_*(-\val_{h+1,t}^+| \epsilon_{h,t}(x,a),\wh{P}_{h,t})$ and $\CB_{h,t}(x,a)= D_*(\val_{h+1,t}^+|\epsilon_{h,t}(x,a),\wh{P}_{h,t})$.
\end{restatable}
The proof is in Appendix~\ref{app:regtab}. While similar results are commonly used in the analysis of value-based algorithms \cite{azar2017minimax,DLB17}, the merit of 
Theorem~\ref{thm:reggen} is that it is derived from a model-optimistic perspective, and thus cleanly separates the probabilistic and algebraic parts of the regret analysis. 
Indeed, proving the probabilistic statement that $P$ is in the confidence set is very simple in the primal space where our constraints are specified. 
Once this is established, the regret can be bounded in terms of the dual exploration bonuses. 
This simplicity of analysis is to be contrasted with the analyses of other value-optimistic 
methods that often interleave probabilistic and algebraic steps in a complex manner. %

\vspace{-5pt}
\paragraph{Inflating the exploration bonus.} 
The downside of the optimistic dynamic-programming algorithm derived above is that the exploration bonuses may sometimes be difficult to 
calculate explicitly. Luckily, it is easy to show that the regret guarantees are preserved if we replace the bonuses by 
an easily-computed upper bound. This is helpful for instance when $D$ is defined as $D(p,p') = \norm{p-p'}$,
whence the conjugate can be simply bounded by the dual norm $\norm{V}_*$. Formally, we can consider
an \emph{inflated conjugate} $D_*^\dag$ satisfying $D_*^\dag(f|\epsilon',\wh{P})\geq D_*(f|\epsilon,\wh{P})$ for every function $f:\cS \to [0,H]$, 
and obtain an optimistic value function by the following dynamic-programming procedure:
\begin{equation}\label{eq:valbound}
\val^\dag_{h}(x) = \max_{a} \bigg\{ \min \bigg\{ H-h+1, r(x,a) + \wh{P}_{h}(\cdot|x,a)\val^\dag_{h+1} + D_*^\dag(\val^\dag_{h+1}| \epsilon'(x,a),\wh{P}_{h}) \bigg \} \bigg \},
\end{equation} 
with $\val^\dag_{H+1}(x)=0 \, \forall x \in \cS$. In this case, we need to clip the value functions since we can no 
longer use Proposition~\ref{prop:valbounds} to show they are bounded.
The resulting value-estimates then satisfy $V^*_1(x_1) \leq \val^+_1(x_1) \leq \val^\dag_1(x_1)$ with high probability, so 
we can bound the regret of this algorithm in the following theorem, whose proof is in Appendix~\ref{app:ubound}:
\begin{restatable}{theorem}{regub}
 \label{thm:regub}
Let $D_*^\dag(f|\epsilon', \wh{P})$ be an upper bound on $D_*(f|\epsilon, \wh{P})$ and $D_*(-f|\epsilon, \wh{P})$ for every $f:\cS \to [0,H]$, 
and, $\CB_{h,t}^\dag(x,a) = D_*^\dag(\val^\dag_{h+1,t}|\epsilon'_{h,t}(x,a),\wh{P}_{h,t})$.
 Then, on the event $\cap_{t=1}^K \{P \in \cP_t\}$, with probability greater than $1-\delta$, the policy returned by the 
 procedure in \eqref{eq:valbound} incurs regret
\[
\Reg_T \leq  2\sum_{t=1}^K \sum_{h=1}^H \textup{CB}_{h,t}^\dag(x_{h,t},\pi_{h,t}(x_{h,t})) +4H\sqrt{2T\log(1/\delta)}. 
\]
\end{restatable}

\begin{table} \label{tab:ex}
\begin{center}
\begin{tabular}{l|c|c|c|c}
\textbf{Algorithm} & \textbf{Distance $D(p,\hp)$} & $\epsilon$ & \textbf{Conjugate $D_*^\dag(\val|\epsilon,\hp)$} & \textbf{Regret} %
\\
\hline
UCRL2 
\cite{jaksch2010near}:& $\onenorm{p-\hp}$ & $\sqrt{S/N}$ & $\epsilon\cdot\text{span}\pa{V}$  & $S\sqrt{H^3AT}$ \\
UCRL2B 
\cite{improved_analysis_UCRL2B}: & $\max_x \frac{(p(x)-\hp(x))^2}{\hp(x)}$ & $1/N$ &  $\sum_x \sqrt{\epsilon \hp(x)} |V(x) - \hp V|$  & $H\sqrt{S\Gamma AT}$\\
KL-UCRL\tablefootnote{In the original KL-UCRL algorithm, \cite{filippi2010optimism,talebi2018variance} consider the reverse KL-divergence. This also fits into our framework. See Appendix~\ref{app:kl} for details.}: & $\sum_x p(x)\log\frac{p(x)}{\hp(x)} + \sum_x(\hp(x)-p(x))$ & $S/N$  &   
$\sqrt{(\epsilon + (1-\sum_y \hp(y))) \wh{\Var}(\val)}$
& $HS \sqrt{ AT}$\\ %
$\chi^2$-UCRL\tablefootnote{\cite{maillard2014hard} also use a $\chi^2$-divergence but require $\wt{P}(x)>p_0$ for some $p_0$ if $\wt{P}(x)>0$ making $\cP$ non-convex.}
& $\sum_x \frac{ (p(y) - \hp(y))^2}{\hp(y)}$ & $S/N$ & $\sqrt{\epsilon \wh{\Var}(\val)}$  & $HS\sqrt{AT}$
\end{tabular}
\end{center}
\caption{Various algorithms in our framework. For all algorithms except UCRL2, we use $\hP^+_{h,t}(y|x,a) = \frac{\max{1,N_{h,t}(x,a,y)}}{N_{h,t}(x,a)}$ as the base measure to avoid division by 0, for UCRL2, we use $\hP(y|x,a)$. We denote $\wh{\Var}(\val) = \sum_x \hp(x)\pa{V(x) - \iprod{\hp}{V}}^2$.
The third column gives scaling of the confidence width in terms of $S$ and the number of sample transitions $N$.
The fourth column gives a tractable upper bound on the value of the conjugate.
The last column gives the %
the regret bound derived from Theorem~\ref{thm:regub} (up to logarithmic factors) with exploration bonus defined 
from the inflated conjugate and the smallest value of $\epsilon$ that guarantees $ \cap_{t=1}^K \{P\in\cP_t\}$ w.h.p.}
\vspace{-10pt}
\end{table}

\vspace{-10pt}
\paragraph{Examples.} Theorems~\ref{thm:reggen}~and~\ref{thm:regub} show that the key quantities governing the size of the 
regret are the conjugate distance and the confidence width $\epsilon$. This explicitly quantifies the impact of the choice of primal confidence set. 
We provide some example choices of the divergences along with their conjugates, the best known confidence widths, and the resulting regret bounds 
in Table~\ref{tab:ex}, with derivations in Appendix~\ref{app:examples}. 
Many of these correspond to existing methods for which
our framework suggests their first dynamic-programming implementation in the original state space $\Sw$, 
rather than the extended state-space which was traditionally used
\citep{jaksch2010near,filippi2010optimism,maillard2014hard}.
More generally, our framework captures any algorithm that defines 
confidence sets in terms of a norm or $f$-divergence, along with many others.
It may also be possible to derive model-optimistic forms of value-optimistic methods, however, in this case care needs to be taken to show that the primal confidence sets are valid. For example, a variant of UCB-VI \citep{azar2017minimax} can be derived from the divergence measure $\<P-\hP,V_{h+1}^+\>$, but the probabilistic analysis here is complicated due to the dependence between $\hP$ and $V^+_{h+1}$.

\section{Optimism with realizable linear function approximation}
We now extend our framework to factored linear MDPs, where all currently known algorithms
are value-optimistic. We provide the first model-optimistic formulation by modeling 
uncertainty about the MDP in the primal LP involving occupancy 
measures in~\eqref{eq:linprimal}.
All proofs are in Appendix~\ref{app:linproofs}.

A key challenge in this setting is that the uncertainty %
can no longer be expressed using
distance metrics %
in the state space, since this could lead 
to trivially large confidence sets\footnote{E.g., for the total variation distance, concentration bounds scale with $\sqrt{S}$ which is potentially unbounded.}.
Instead, we define confidence sets in terms of a distance that takes 
the linear structure
into account. 
These are centered around a reference model $\wh{P}$ defined for 
each $h,a$ as $\hP_{h,a} = \Phi \wh{M}_{h,a}$ for some $d\times \Sw$ matrix $\wh{M}_{h,a}$. %
We consider reference models implicitly defined by the LSTD algorithm \cite{BraBa96,lagoudakis2003least,PalITaPWLi08}. %
In episode $t$, let
$\Sigma_{h,a,t} = \sum_{k=1}^t \II{a_{h,k} = a} \varphi(x_{h,k})\varphi\transpose(x_{h,k}) + \lambda I$ 
for some $\lambda \ge 0$, 
and $\be_x$ be the unit vector in $\real^\Sw$ corresponding to 
state $x$. Then, our reference model in episode $t$ is defined for each action $a$ as
\begin{equation}\label{eq:Mhat}
 \wh{M}_{h,a,t} = \Sigma_{h,a,t-1}^{-1} \sum_{k=1}^{t-1} \II{a_{h,k} = a} \varphi(x_{h,k}) \be_{x_{h+1,k}}.
\end{equation}
Finally, the weight matrix in the LP formulation~\eqref{eq:linprimal} is chosen as 
$W_{h,a,t} = \sum_{k=1}^t \II{a_{h,k}=a} \be_{x_{h,k}} \be_{x_{h,k}}\transpose$, so that
 $\Phi\transpose W_{h,a,t} \Phi = \Sigma_{h,a,t} - \lambda I$.
We establish the following important technical result: %
\begin{restatable}{proposition}{conc_lin}\label{prop:conc_lin}
Consider the reference model $\wh{P}_{h,a,t} = \Phi \wh{M}_{h,a,t}$ with $\hM_{h,a,t}$ defined in Equation~\eqref{eq:Mhat}.
Then, for any fixed function $g:\Sw\ra [-H,H]$, the following holds with probability at least $1-\delta$:
\[
\bigl\|\bpa{M_{h,a,t} - \hM_{h,a,t}}g\bigr\|_{\Sigma_{h,a,t-1}} \le H\sqrt{d \log\pa{\frac{1 + t R^2/\lambda}{\delta}}} + C_P H \sqrt{\lambda d}.
\] 
\end{restatable}
\vspace{-5pt}
The proof is based on the fact that for a fixed $g$, $\bpa{M_{h,a,t} - \hM_{h,a,t}}g$ is 
essentially a vector-valued martingale. %
Our main contribution in this %
 setting is to use this result to identify two distinct ways of 
deriving tight confidence sets that incorporate optimism into~\eqref{eq:linprimal}. 
Both approaches use the optimistic parametric Bellman (OPB) equations with some exploration bonus $\CB_{h,t}(x,a)$ (defined later):%
\begin{equation}\label{eq:bellman_opt_lin}
\begin{split}
 \theta^+_{h,a,t} &= \rho_a + \Sigma_{h,a,t-1}^{-1} \sum_{k=1}^{t-1} \II{a_{h,k} = a} \varphi(x_{h,k}) \val^+_{h+1,t} \pa{x_{h+1,k}}
 \\
 \val^+_{h,t}(x) &= \max_{a} \ev{\pa{\Phi\theta^+_{h,a,t}}\pa{x} + \CB_{h,t}(x,a)} 
\end{split}
\end{equation}
Both bonuses we derive can be upper-bounded by $ \CB^\dag_{h,t}(x,a) = C(d) \norm{\varphi(x)}_{\Sigma_{h,a,t-1}^{-1}}$ for some %
$C(d)>0$.
Then, one can apply a variant Theorem~\ref{thm:regub} to bound the regret of both algorithms %
in terms of the sum of these inflated exploration bonuses, amounting to a total regret of $\wt{O}(C(d)\sqrt{dHT})$. 

\subsection{Optimism in state space through local confidence sets}
Our first approach %
models the uncertainty locally in each 
state-action pair $x,a$ using some distance metric $D$ between transition functions.
We consider the following optimization problem:
\begin{gather} \label{eq:linopt1}
\begin{array}{ll}
&\underset{q \in \cQ(x_1), \omega}{\text{maximize}}  \vspace{5pt}
\\  \vspace{15pt} &\sum_{h=1}^H \sum_a\iprod{ W_{h,a}\Phi \omega_{a,h}}{r_a}
\end{array}
 \left| 
 \hspace{-10pt}
 \begin{array}{rll}
 &\text{subject to}\qquad \qquad\quad 
 \\
 &\sum_a q_{h+1,a} = \sum_a \wt{P}_{h,a} W_{h,a} \Phi \omega_{h,a} \qquad &\forall h\in[H]
 \\
 &\Phi\transpose q_{h,1} = \Phi\transpose W_{h,a} \Phi \omega_{h,a} \qquad &\forall a \in \cA, h\in[H] 
 \\
 &D\pa{\wt{P}_{h}(\cdot|x,a), \wh{P}_{h}(\cdot|x,a)} \leq \epsilon_h(x,a) \, &\forall (x,a) \in \cZ, h \in [H]
 \end{array}\right.\raisetag{1.8cm}
\end{gather}
As in the tabular case, %
\eqref{eq:linopt1} can be reparametrized so that the constraint set is convex,  %
allowing us to appeal to Lagrangian duality
to get an equivalent formulation as shown in the following proposition.
\begin{restatable}{proposition}{dual_approx_1}\label{prop:dual_approx_1}
The optimization problem~\eqref{eq:linopt1} is equivalent to solving the optimistic Bellman equations~\eqref{eq:bellman_opt_lin} 
with the exploration bonus defined as
$ \CB_{h}(x,a) = D^*(\val^+_{h+1} | \epsilon_h(x,a),\wh{P}_h(\cdot|x,a))$.
\end{restatable}
Taking the form of %
$\val^+_h$ into account, in episode $t$,
we define our confidence sets as in \eqref{eq:Pdef} with %
\begin{equation}\label{eq:CBonus_lin}
 D\pa{\wt{P}_{h,t}(\cdot|x,a), \wh{P}_{h,t}(\cdot|x,a)} = \sup_{g\in\V_{h+1,t}} \sum_{x'}\pa{\tP_{h,t}(x'|x,a) - \hP_{h,t}(x'|x,a)} g(x') %
\end{equation}

\vspace{-10pt}
and %
$\wh{P}_{h,a,t}=\Phi\wh{M}_{h,a,t}$ %
where $\V_{h+1,t}$ is the set of value functions that can be produced 
by solving the OPB equations~\eqref{eq:bellman_opt_lin}. 
For any choice of $\epsilon_t$, %
$\CB_{h,t}(x,a) \le \epsilon_{h,t}(x,a)$, 
so one can simply use the bonus $\CB_{h,t}^\dag(x,a) = \epsilon_{h,t}(x,a)$.
The following theorem 
bounds the regret for an appropriate choice of $\epsilon_t$
\begin{restatable}{theorem}{lem:bonus_lin}\label{lem:bonus_lin}
The choice $\epsilon_{h,t}(x,a) = C \norm{\varphi(x)}_{\Sigma_{h,a,t-1}^{-1}}$ with $C = \wt{O}(Hd)$ guarantees that 
the transition model $P$ is feasible for~\eqref{eq:linopt1} in every episode $t$ with probability $1-\delta$. The resulting 
optimistic algorithm with exploration bonus $\CB^\dag_{h,t}(x,a) = \epsilon_{h,t}(x,a)$ has regret bounded by $\wt{O}(\sqrt{H^3d^3T})$.
\end{restatable}
\vspace{-3pt}
This algorithm coincides with the LSVI-UCB method of \cite{jin2020provably} and our performance guarantee 
matches theirs. %
The advantage of our result is a %
simpler analysis allowed by our model-optimistic perspective.

\vspace{-5pt}
\subsection{Optimism in feature space through global constraints}\label{sec:global_feat_opt}
\vspace{-5pt}
Our second approach %
exploits the structure of the reference model~\eqref{eq:Mhat}, 
and constrains %
$\wt{P}_a$ through global conditions on $\tM_a$. 
We define $\cP_t$ using %
the distance metric suggested by Proposition~\ref{prop:conc_lin} as
\begin{equation}\label{eq:global_const}
 D(\wt{M}_{h,a},\wh{M}_{h,a}) = \sup_{f\in\V_{h+1}} \bnorm{\bpa{\tM_{h,a} - \hM_{h,a}}f}_{\Sigma_{h,a}} \le \epsilon_{h,a}
\end{equation}
\vspace{-5pt}
for $\V_{h+1}$ as in \eqref{eq:CBonus_lin} and some $\epsilon_{h,a} >0$.
We then consider the following %
optimization problem:
\begin{gather}\label{eq:featOP}
\underset{\substack{q \in \cQ(x_1),\\ \omega ,\wt{M}}}{\text{maximize}} \quad \sum_{h=1}^H \sum_a\iprod{ W_{h,a}\Phi 
\omega_{h,a}}{r_a}
 \left| 
 \begin{array}{rll}
& \text{subject to}\qquad \qquad\quad 
 \\
 &\sum_a q_{h+1,a} = \sum_a \wt{P}_{h,a} \transpose W_{h,a} \Phi \omega_{h,a} \qquad &\forall h\in[H]
 \\
 &\Phi\transpose q_{h,a} = \Phi\transpose W_{h,a} \Phi \omega_{h,a} \qquad &\forall a \in \cA, h\in[H] 
 \\
& D(\wt{M}_{h,a},\wh{M}_{h,a}) \le \epsilon_{h,a} \qquad & \forall a \in \cA, h \in [H].
 \end{array}\right.\raisetag{1.8cm}
 \vspace{-5pt}
\end{gather}

Unfortunately, directly constraining $M$ leads to an optimization problem that, unlike in the other settings, cannot easily be re-written as an convex problem exhibiting strong duality. 
Nevertheless, %
for a fixed $\tM$, the value of \eqref{eq:featOP} %
is equivalent to
$G(\tM) = V_1^+(x_1)$ where $V^+$ solves the OPB equations~\eqref{eq:bellman_opt_lin} with %
 $ \CB_{h}(x,a) = \biprod{\varphi(x)}{\bpa{\tM_{h,a} - \hM_{h,a}} V_{h+1}}$.
Let $\M = \{ \wt{M} \in \bR^{d \times S}: D(\wt{M},\wh{M}) \le \epsilon \}$,
then, we can re-write %
\eqref{eq:featOP} as maximizing $G(\tM)$ over $\tM\in\M$.
Exploiting this %
we provide a more tractable version of the optimization problem, and bound the regret of the resulting algorithm, 
below: %

\begin{theorem}\label{prop:regfeat}
Define the function 
$G'(B) = V_1^+(x_1)$ with $V^+$ the solution
 of the OPB equations~\eqref{eq:bellman_opt_lin} with exploration bonus
  $\CB_{h}(x,a) = \iprod{\varphi(x)}{B_{h,a}}$
and let $\mathcal{B}_t = \bigl\{B: \norm{B_{h,a}}_{\Sigma_{h,a,t-1}} \le \epsilon_{h,a,t}\bigr\}$ for all episodes $t \in [K]$.
Then, %
$\max_{B\in\mathcal{B}_t} G'(B) \ge \max_{\tM \in \M_t} G(\tM)$ and the
optimistic algorithm with exploration bonuses corresponding to the optimal solutions $B^\dag_t$ 
has regret bounded by $\wt{O}(d\sqrt{H^3 T})$.
\end{theorem}
The algorithm suggested in this theorem essentially coincides with the \eleanor method proposed recently in 
\cite{zanette2020learning}, and our guarantees match theirs under our realizability assumption.
Our model-based perspective suggests that the problem of implementing \eleanor is inherently hard:
the form of the primal optimization problem reveals that $G'(B)$ is a convex function of $B$, 
and thus its maximization over a convex set is intractable in general. 
Note that the celebrated LinUCB algorithm for linear bandits
must solve the a similar convex maximization 
problem~\citep{DHK08,abbasi2011improved}. 
As in linear bandits,
it remains an open question to get regret $\wt{O}(Hd\sqrt{T})$ 
with a computationally efficient algorithm. %

\vspace{-5pt}

\section{Conclusion}
We have provided a new framework unifying model-optimistic and value-optimistic approaches for episodic reinforcement learning, thus demonstrating that many desirable features are enjoyed by both approaches. 
In the tabular setting, we provided improved implementations and analyses of a general class of model-optimistic algorithms. 
While these results demonstrate the strength and flexibility of the model-based perspective, our regret bounds feature an additional factor 
of $\sqrt{S}$ on top of the minimax optimal bounds, which has been eliminated by value-optimistic methods \citep{azar2017minimax,DLB17}. 
However, our bounds for factored linear MDPs match the best existing results, which gives us hope that model-based approaches may also eventually prove to be optimal in the tabular case.
Finally, we note that it is straightforward to extend our framework for infinite-horizon MDPs, although we leave the challenge of analyzing the regret of the resulting algorithms for future work.

\bibliographystyle{abbrvnat}
\bibliography{lprl_refs}

\newpage
\appendix
\begin{center}
{\Large \textbf{Appendix}}
\end{center}

\section{Proofs of Results for Tabular Setting}
We prove here the results of Section~\ref{sec:tab}. For ease of exposition, we restate the results before proving them. 
For convenience, we introduce the confidence set for every state $x \in \cS$, action $a \in \cA$ and stage $h \in [H]$,
\begin{align}%
\cP_h(x,a) = \ev{ \wt{P}_h(\cdot|x,a) \in \Delta: D\pa{\tP_h(\cdot|x,a) , \wh{P}_h(\cdot|x,a)} \leq \epsilon(x,a) }
\label{eq:Pxadef}
\end{align}
and note that $\wt{P} \in \cP$ if $\wt{P}_h(\cdot|x,a) \in \cP_h(x,a)$ for all $x,a,h$ %

The following lemma will be useful in several of the proofs.
\begin{lemma} \label{lem:strongcomp}
The primal and dual optimization problems in \eqref{eq:primalMDPopt} and \eqref{eq:dualtab} exhibit strong duality. Consequently the Karush-Kuhn-Tucker (KKT) conditions hold, and in particular, complementary slackness holds.
\end{lemma}
\begin{proof}
We first show that the optimization problem in~\eqref{eq:primalMDPopt} exhibits strong duality.
For this, it is helpful to consider a reparameterization where we introduce the variables $J_h(x,a,x') = \wt{P}_h(x'|x,a)q_h(x,a)$, so that the non-convex constraint $D\bpa{\wt P_h(\cdot|x,a), \wh{P}_h(\cdot|x,a)} \leq \epsilon(x,a)$ can be rewritten as $D\bpa{J_h(x,a,\cdot), \wh{P}_h(\cdot|x,a)q_h(x,a)} \leq \epsilon_h(x,a)q_h(x,a)$, which is convex in $J$ and $q$. The two constraints are clearly equivalent due to positive homogeneity of $D$. This implies that the optimization problem in~\eqref{eq:primalMDPopt} can be equivalently written as %
\begin{align}
 \underset{q\in \cQ(x_1), J}{\text{maximize}}  \;& \sum_{x,a,h} q_h(x,a) r(x,a) \label{eq:pseudoprimal}
\\ \text{Subject to} \; & \sum_a q_h(x,a) = \sum_{x',a'} J_{h-1}(x',a',x) \quad & \forall x \in \cS, h \in [H]  \nonumber
\\ & D\pa{J_h(x,a,\cdot), \wh{P}_h(\cdot|x,a)q_h(x,a)} \leq \epsilon_h(x,a)q_h(x,a) \quad & \forall (x,a) \in \cZ, h \in [H] \nonumber
\\ & \sum_{x'} J_h(x,a,x') =q_h(x,a) \quad & \forall (x,a) \in \cZ, h \in [H] \nonumber
\\& J_h(x,a,x') \geq 0 \quad & \forall x,x' \in \cS, a \in \cA, h \in [H]\nonumber.
\end{align}
 In this formulation, there is only one non-linear constraint, and by our assumption that $D$ is convex in both of its arguments, 
 this constraint is convex in $J$ and $q$. Moreover, $\wh{J}_h(x,a,x') = \wh{P}_h(x'|x,a)q_h(x,a)$ satisfies this constraint for any $q_h(x,a)$, and in particular, if $q_h(x,a)$ is the occupancy measure induced by any policy $\bpi$ in the MDP with transition function $\wh{P}$, then $q_h(x,a)$ and $J_h(x,a,a')$ are feasible solutions to the primal. 
 Hence, the Slater conditions are satisfied, and thus the optimization problem exhibits strong duality (see e.g. \cite{boyd2004convex}).
 We can then write the dual of the optimization problem in \eqref{eq:pseudoprimal} as
 \begin{align}
\max_{(q,M) \in \cC_1} \min_{\val,\gamma}& \bigg \{ \sum_{x,a,h} q_h(x,a) (\val_{h}(x) - \gamma_h(x,a)+r(x,a) + \val_1(x_1)\label{eq:pseudodual}
	\\ & \hspace{120pt} + \sum_{x,a,x',h} J_h(x,a,x') (\val_{h+1}(x') + \gamma_h(x,a)) \bigg \}, \nonumber
 \end{align}
 where $\cC_1 = \bev{ q,J :  D(J_h(x,a,\cdot), \wh{P}_h(\cdot|x,a)q_h(x,a)) \leq \epsilon_h(x,a)q_h(x,a) \ \ (\forall x,a)}$.
 Then, we can use the reverse reparameterization to rewrite this in terms of $\wt{P}_h(x'|x,a) = J_h(x,a,x')/q_h(x,a)$, noting that $\wt P_h(\cdot|x,a)$ is a valid probability density by constraints on $J,q$. We get,  %
 \begin{align}
&\max_{q,\wt{P} \in \cP} \min_{\val,\gamma} \bigg \{ \sum_{x,a,h} q_h(x,a) (-\val_{h}(x) - \gamma_h(x,a)+r(x,a)) + \val_1(x_1) \nonumber
	\\ & \hspace{120pt} + \sum_{x,a,x',h} \wt{P}_h(x'|x,a) q_h(x,a) (\val_{h+1}(x') + \gamma_h(x,a)) \bigg \} \nonumber
\\ & = \max_{q,\wt{P} \in \cP} \min_{\val} \bigg \{ \sum_{x,a,h} q_h(x,a) \bigg(-\val_{h}(x) +r(x,a) + \sum_{x'} \wt{P}_h(x'|x,a)\val_{h+1}(x')  \bigg) + \val_1(x_1) \bigg \}, \label{eq:duallag} 
 \end{align}
 where $\cP = \bev{ \wt{P} \in \Delta: D(\wt{P}_h(\cdot|x,a), \wh{P}_h(\cdot|x,a)) \leq \epsilon_h(x,a) \ \ (\forall x,a,h)}$, and the last equality follows since $\sum_y \wt P_h(y|x,a)=1$. This is the Lagrangian dual form of the original optimization problem we considered. 
 Let $\text{OBJ}(a)$ denote the objective function of the optimization problem in equation $(a)$. It then follows that,
 \[ \text{OBJ}\eqref{eq:primalMDPopt} = \text{OBJ}\eqref{eq:pseudoprimal} = \text{OBJ} \eqref{eq:pseudodual} = \text{OBJ}\eqref{eq:duallag} \]
 and so strong duality holds for the problem in \eqref{eq:primalMDPopt}.
 Thus, by standard results (e.g., \cite[Section 5.5.3]{boyd2004convex}), we conclude that the KKT conditions are satisfied by $(q^+, \wt{P}^+,\val^+)$, the optimal solutions to the primal and dual. As a consequence, complementary slackness also holds. This concludes the proof.
\end{proof}

\subsection{Duality Result} \label{app:dual}
\dual*
\begin{proof}
It will be helpful to write the primal optimization problem as
\begin{align*}
& \underset{q\in\cQ(x_1), \tilde P, \kappa}{\text{maximize}}  \; \sum_{x,a,h} q_h(x,a) r(x,a) 
\\ &\text{Subject to}
\\ & \sum_a q_h(x,a) = \sum_{x',a'} \hat P_h(x|x',a')q_h(x',a') + \sum_{x',a'} \kappa_h(x',a',x) q_h(x',a')  \hspace{-20pt}& \forall x \in \cS, h \in [H]
\\& \kappa_h(x,a,x') = \tilde P_h(x'|x,a)- \wh{P}_h(x'|x,a) & \forall x,x' \in \cS, a \in \cA, h \in [H]
\\ & D\pa{\wh{P}_h(\cdot|x,a), \wh{P}_h(\cdot|x,a)} \leq \epsilon_h(x,a)  & \forall (x,a) \in \cZ, h \in [H]
\\ & \sum_{x'} \kappa_h(x,a,x') =0 & \forall (x,a) \in \cZ, h \in [H].
\end{align*}
By Lemma~\ref{lem:strongcomp}, we know that this problem exhibits strong duality.
We then consider the partial Lagrangian of the above problem without the constraints on $\wt{P}$, which yields
\[ \cL(q,\kappa; \val) = \sum_{x,a,h} q_h(x,a) \bigg( \sum_y \hat P_h(y|x,a) \val_{h+1}(y) + \sum_y \kappa_h(x,a,y) \val_{h+1}(y) +r(x,a) - \val_h(x) \bigg) + \val_1(x_1)\]
For $\cP$ defined in \eqref{eq:Pdef}, we know that the optimal value of the objective function of the primal optimization problem is given by the Lagrangian relaxation,
\[ \min_{\val} \max_{q \geq 0,\kappa,\wt{P} \in \cP} \cL(q,\kappa;\val).\]
To proceed, we fix a $\val$ and consider the inner maximization problem. By definition of $\kappa_h(x,a,x') = \wt{P}_h(x'|x,a) - \wh{P}_h(x'|x,a))$, we can write
\begin{align}
&\max_{q \geq 0,\kappa, \wt{P} \in \cP} \cL(q,\kappa;\val)  \nonumber
\\ &= \max_{q \geq 0,\kappa, \wt{P} \in \cP} \sum_{x,a,h} q_h(x,a) \bigg( \sum_y \wh P_h(y|x,a) \val_{h+1}(y) + \sum_y \kappa_h(x,a,y) \val_{h+1}(y) +r(x,a) - \val_h(x) \bigg) \nonumber
	\\ & \hspace{300pt} + \val_1(x_1) \nonumber
\\ &= \max_{q\geq 0} \sum_{x,a,h} q_h(x,a) \bigg( \sum_y \wh P_h(y|x,a) \val_{h+1}(y) + \max_{\substack{\kappa_h(x,a,\cdot) \\ \wt{P}_h(\cdot|x,a) \in \cP_h(x,a)}} \sum_y \kappa(x,a,y) \val_{h+1}(y) +r(x,a) \nonumber
	\\ & \hspace{300pt} - \val_h(x) \bigg) + \val_1(x_1) \nonumber
\\ & = \max_{q\geq 0} \sum_{x,a,h} q_h(x,a) \bigg( \sum_y \wh P_h(y|x,a) \val_{h+1}(y) + D_*(\val_{h+1} |\epsilon_h(x,a) , \wh{P}_h(\cdot|x,a)) +r(x,a) - \val_h(x) \bigg)  \nonumber
	\\ & \hspace{300pt}+ \val_1(x_1)\label{eqn:Lmax},
\end{align}
where $\cP_h(x,a)$ is the set in \eqref{eq:Pxadef}. The second equality crucially uses that $q_h(x,a)\ge 0$ and the last equality follows from the definition of the conjugate $D_*$:
\begin{align*}
&\max_{\kappa_h(x,a,\cdot),\wt{P}_h(\cdot|x,a)) \in \cP_h(x,a)} \sum_y \kappa_h(x,a,y) \val_{h+1}(y)
\\ &\qquad\qquad = \max_{\wt{ P}_h(\cdot|x,a) \in \Delta} \bigg\{ \biprod{\wt{P}_h(\cdot|x,a) - \wh{P}_h(\cdot|x,a)}{\val_{h+1}}; D(\wt{P}_h(\cdot|x,a),\wh{P}_h(\cdot|x,a)) \leq \epsilon_h(x,a) \} %
\\ &\qquad\qquad= D_*(\val_{h+1}|\epsilon_h(x,a), \wh{P}_h(\cdot|x,a)).
\end{align*}

We then optimize the expression in \eqref{eqn:Lmax} with respect to $q$ and $\val$ using an adaptation of techniques used for establishing LP duality between the original problems~\eqref{eq:V_LP} and~\eqref{eq:primal_full}. Specifically, let $g(\val) = \max_q \cL(q; \val)$ and note that by~\eqref{eqn:Lmax}, the Lagrangian no longer depends on $\kappa$ or $\wt{P}$. Then, define $\eta_h(x,a) = \sum_y \wh{P}_h(y|x,a) \val_{h+1}(y) + D_*(\val_{h+1}|\epsilon_h(x,a), \wh{P}_h(\cdot|x,a)) +r(x,a) - \val_h(x) $ for all $x,a,h$ and observe that
\begin{align*}
 g(\val) = \val_1(x_1) + \max_q \sum_{x,a,h} q_h(x,a) \eta_h(x,a) = 
 \begin{cases} \val_1(x_1)  & \text{ if } \eta_h(x,a) \leq 0 \quad \forall x,a,h
 \\ \infty & \text{ otherwise}.
 \end{cases}
 \end{align*}
Thus, we can then write the dual optimization problem of minimizing $g(\val)$ with respect to $\val$ as
\begin{align*}
\underset{\val}{\text{minimize }}  \;& \val_1(x_1)
\\ \text{Subject to } \;& \val_h(x) \geq r(x,a) + \sum_y \wh{P}_h(y|x,a) \val_{h+1}(y) +  D_*(\val_{h+1}|\epsilon_h(x,a), \wh{P}_h(\cdot|x,a)).
\end{align*}
This proves the proposition.
\end{proof}

\subsection{Properties of the Optimal Solutions} \label{app:properties}
In this section we prove Propositions~\ref{prop:pols}~and~\ref{prop:valbounds}.
In order to prove Proposition~\ref{prop:pols}, we first need the following result which gives the form of the optimal solution to the dual in Equation~\eqref{eq:dualtab}.
\begin{lemma} \label{lem:dualform}
 The solution to the dual in \eqref{eq:dualtab} is given by
\begin{align}
\val^+_h(x) = \max_{a \in \cA} \bigg\{ r(x,a)+ \CB_{h}(x,a)  + \sum_{y \in \cS} \hP_h(y|x,a) \val_{h+1}^*(y) \bigg \}\label{eqn:betastardef2}
\end{align}
where we use the notation $\CB_{h}(x,a)=D_*(\val_{h+1}|\epsilon_h(x,a), \wh{P}_h(\cdot|x,a)))$.
\end{lemma}
\begin{proof}
The structure of the constraints on $\val_h(x)$ in \eqref{eq:dualtab} and the definition of $\CB_h(x,a)$ mean that $\val^+_h(x)$ can be determined using only the values of $\val^+_{l}$ for $l \geq h+1$.
Hence, we can prove the result by backwards induction on $h=H,\dots, 1$. For the base case, when $h=H$, the constraint in the dual is
\[ \val_H(x) \geq r(x,a) + \CB_H(x,a)  \quad \forall x \in \cS, a \in \cA.\]
In order to minimize $\val_H(x)$, we set $\val^+_H(x) = \max_{a \in \cA}\{ r(x,a) + \CB_H(x,a)\}$ for all $x \in \cS$. Now assume that for stage $h+1$, the optimal value of $\val^+_{h+1}(x)$ is given by \eqref{eqn:betastardef2}. Then, when considering stage $h$, we wish to set $\val^+_h(x)$ as small as possible. By the inductive hypothesis, we know it is optimal to set $\val_{h+1}(x)=\val^+_{h+1}(x)$, and we know that $\text{CB}_{h}(x,a)$ has been defined using only terms from stage $h+1$ and is minimal. Consequently, the RHS of the constraint in \eqref{eq:dualtab} is minimized for any $(x,a,h)$ by setting $\val_{h+1}=\val^+_{h+1}$. This means that the minimal value of $\val_h$ is given by \eqref{eqn:betastardef2}. Hence the result holds for all $h=1,\dots, H$, and so considering $h=1$ and initial state $x_1$, we can conclude that $\val^+$ is the optimal solution to the LP in \eqref{eq:dualtab}.
\end{proof}

We now prove Proposition~\ref{prop:pols}.
\pols*
\begin{proof}
By Lemma~\ref{lem:dualform}, we know that the optimal solution to the dual in \eqref{eq:dualtab} is given by
\begin{align}
\val^+_h(x) = \max_{a \in \cA} \bigg\{ r(x,a)+ \text{CB}_{h}(x,a)  + \sum_{y \in \cS} \hP_h(y|x,a) \val_{h+1}^+(y) \bigg \}\label{eqn:betastardef}.
\end{align}
 
We then proceed by considering the case where the right hand side of the expression in \eqref{eqn:betastardef} has a unique maximizer. In this case, let 
 \[a_h^*(x) = \argmax_{a \in \cA} \bigg\{ r(x,a)+ \text{CB}_{h}(x,a)  + \sum_{y \in \cS} \hP_h(y|x,a) \val_{h+1}^+(y) \bigg \}. \]
 Since $a_h^*(x)$ is the unique maximizer of this expression, it follows that, for a fixed $x,h$, the constraint in \eqref{eq:dualtab} is only binding for one $a \in \cA$, namely $a_h^*(x)$. By Lemma~\ref{lem:strongcomp}, we know that complementary slackness holds for this problem. Then, using complementary slackness, it follows that only one of the primal variables is non-zero. 
In particular, for a fixed state $x$ and stage $h$, $q^+_h(x,a)=0$ for all $a \neq a^*_h(x), x' \in \cS$. Consequently, $\pi^+(x) =a_h^*(x)$ and so the policy induced by $q^+$, $\pi^+$, will only have non-zero probability of playing the action which maximize the right hand side of \eqref{eqn:betastardef}.
 
We now consider the case where there are multiple maximizers of the right hand side of \eqref{eqn:betastardef}. Let $a_h^1(x), \dots, a_h^m(x)$ denote the $m$ maximizers. By a similar argument to the previous case, we know that for a fixed $x \in \cS$ and $h \in [H]$, the constraint in \eqref{eq:dualtab} is only binding for $a=a_h^i(x)$ for some $i \in [m]$. Then, by complementary slackness, it follows that $q^+_h(x,a)=0$ for all $a \neq a^*_i(x)$ for $i \in [m]$, and so the only non-zero values of $q^+_h(x,a)$ can occur for $a=a_h^i(x)$ for some $i \in [m]$. The action chosen from state $x$ by policy $\pi^+$ must be one of the actions for which $q^+_h(x,a)>0$ by properties of the relationship between occupancy measures and policies. Hence, $\pi^+(x) = a_h^i(x)$ for some $i \in [m]$, and so equation~\eqref{eq:optpolicy} must hold.
\end{proof}

\valbounds*
\begin{proof}
We begin by proving that if $P \in \cP$, then $V^*_h(x) \le \val^+_h(x) $.

Let $q^*$ be the occupancy measure corresponding to the optimal policy $\pi^*$ under $P$. Then, if $P \in \cP$, then $P$ must feasible for the primal in \eqref{eq:primalMDPopt}, and so it must be the case that 
\[ \sum_{x,a} \sum_{h=1}^H r(x,a) q_h^* (x,a) \leq \sum_{x,a} \sum_{h=1}^H r(x,a) q^+_h(x,a), \]
where $q^+$ is the optimal solution to the LP in \eqref{eq:primalMDPopt}. Considering the LHS of this expression, and the fact that $ q^*$ is the occupancy of the optimal policy $\pi^*$ under the true transition function, it follows that 
\[ \sum_{x,a} \sum_{h=1}^H r(x,a) q_h^*(x,a)= \E \bigg[ \sum_{h=1}^H r(X_h, \pi^*(X_h)) \bigg| X_1=x_1 \bigg] = V^*_1(x_1) \] %
Hence, when $P \in \cP$,
\[ V^*_1(x_1) \leq  \sum_{x,a} \sum_{h=1}^H r(x,a) q^+_h(x,a) = \val_{1}^+(x_1).  \]
for the initial state $x_1$, where we have used the fact that the value of the optimal objective functions are equal due to strong duality (Lemma~\ref{lem:strongcomp}).

In order to prove the result for $x\neq x_1$ and $h\neq 1$, we consider modified linear programs defined by starting the problem at stage $h$ with all prior mass in state $x$. In this case, define the initial state as $x_h=x$, the we write the modified primal optimization problem as 
\begin{align}
 \underset{q \in  \cQ(x), \wt{P} \in \Delta}{\text{maximize}}  &\sum_{l=h}^H \sum_{x,a} q_lx,a)r(x,a) \label{eqn:pseudoprimal}
 \\ \text{Subject to } & \sum_{a \in \cA} q_l(x,a) = \sum_{x' \in \cS, a' \in \cA} \wt{P}_l(x|x'a') q_{l-1}(x',a') \qquad &\forall x \in \cS,\,l=h+1,\dots,H  \nonumber
 \\ & D\pa{\wt{P}_l(\cdot|x,a), \wh{P}_l(\cdot|x,a)} \leq \epsilon_l(x,a), \qquad &\forall (x,a) \in \cS, l =h+1, \dots, H \nonumber
\end{align}
where $\cQ(x)$ has been modified to account for the new initial state. 
Observe that this problem is analogous to the primal optimization problem in \eqref{eq:primalMDPopt}, and hence we can apply the same techniques as used to prove Proposition~\ref{prop:dual} to show that the dual can be written as
\begin{align}
\underset{\val}{\text{minimize}}  \;& \val_h(x)&  \label{eqn:pseudodual}
 \\ \text{ subject to } \;& 
  \val_{l}(x) \geq r(x,a)+ \text{CB}_{l}(x,a)  + \sum_{y \in \cS} \hP_{l}(y|x,a) \val_{l+1}(y)  & \forall (x,a)\in\cS\times\cA, l\in[h:H]\nonumber.
\end{align}
where $\text{CB}_l(x,a) = D_*(\val_{l+1}|\epsilon_l(x,a), \wh{P}_l(\cdot|x,a))$.
Analyzing this dual shows that for $l=h,\dots, H$ and $x \in \cS$, the constraints on $\val_{h}(x)$ here are the same as those in the full dual in \eqref{eq:dualtab}. This means that the dual in \eqref{eq:dualtab} can be broken down per stage and the optimal solution can be found by a dynamic programming style algorithm. In particular, the optimal solution $\val^+_{h}(x)$ in the complete dual in \eqref{eq:dualtab} is given by the optimal value of the objective function in the optimization problem in \eqref{eqn:pseudodual}. 
Note that strong duality also applies in this modified problem since the technique used to prove this in Lemma~\ref{lem:strongcomp} also applies here.
We therefore know that $\val^+_{h}(x) = \sum_{x,a} \sum_{l=h}^H \tilde q^+_l(x,a)r(x,a)$ where $\tilde q^+$ is the optimal solution to the modified LP in \eqref{eqn:pseudoprimal}. On the event that $P$ is in the confidence set, the occupancy measure $\tilde q^*$ defined by the optimal policy $\pi^*$ and $P$ starting from state $x$ in stage $h$ must be a feasible solution to the LP in \eqref{eqn:pseudoprimal}. Consequently, by the same argument as before,
\[ V^*_h(x) =  \sum_{x,a} \sum_{l=h}^H r(x,a) \tilde q_l^* (x,a) \leq \sum_{x,a} \sum_{l=h}^H r(x,a) \tilde q^+_l(x,a) = \val^+_h(x), \]
thus proving the first inequality in the statement of the proposition for all $(x,a) \in \cZ, h=1,\dots, H$.

We now show that $\val^+_h(x) \le H-h+1$ for all $x\in\cS, h \in [H]$. The proof is similar to the previous case and again relies on building a new MDP from each state $x$ in stage $h$ and considering the dual. In particular, for any $x \in \cS, h \in[H]$, in the dual LP in \eqref{eqn:pseudodual}, we see that that the optimal solution to the objective function has value $\val^+_{h}(x)$. By strong duality, this must have the same value as $\sum_{x,a} \sum_{l=h}^H \tilde q^+_l(x,a) r(x,a)$, the optimal value of the objective function of the primal optimization problem in \eqref{eqn:pseudoprimal} started at $x$ in stage $h$. The optimal solution $\tilde q^+$ must be a valid occupancy measure since by the primal constraints $q_l(x,a) \geq 0$ and $\sum_{x,a} q_h(x,a)=1$ are satisfied. It also follows that $\sum_{x,a} q_l(x,a)=1$ for all $l=h+1,\dots, H$ by Lemma~\ref{lem:primalsum}.
From this it follows that $\tilde q_l^+(x,a) \leq1, \, \forall (x,a) \in \cZ, l=h,\dots, H$ so combining this with the fact that $r(x,a) \in [0,1] \, \forall (x,a) \in \cZ$, it must be the case that  $\sum_{x,a} \sum_{l=h}^H \tilde q^+_l(a,x)r(x,a) \leq H-h+1$, and so $\val^+_{h,t}(x) \leq H-h+1$ and the result holds. 
\end{proof}

\begin{lemma} \label{lem:primalsum}
For any feasible solution $q$ to the primal problem in \eqref{eq:primalMDPopt}, it must hold that $\sum_{x,a}q_h(x,a)=1$ for all $h \in [H]$.
\end{lemma}
\begin{proof}
The proof follows by induction on $h$. For the base case, when $h=1$, 
\[ \sum_{x,a} q_1(x,a) = \sum_a q_1(x_1,a) =1 \]
by the constraint $\sum_a q_1(x,a)= \one \{x=x_1\}$ for all $x \in \cS$. Now assume the result holds for $h$, and we prove it for $h+1$. By the flow constraint (first constraint in \eqref{eq:primalMDPopt}), for any feasible $\wt{P} \in \cP$,
\[ \sum_{x,a} q_{h+1}(x,a)= \sum_x \bigg( \sum_{x',a'} \wt{P}_h(x|x',a')q_h(x',a') \bigg) = \sum_{x',a'} q_h(x',a') =1 \]
since $\sum_x \wt{P}_h(x|x',a')=1$. Thus the result holds for all $h=1,\dots,H$.
\end{proof}

\subsection{Regret Bounds} \label{app:regtab}
In this section, we bound the regret of any algorithm that fits into our framework. 

\reggen*
\begin{proof}
The proof is similar to standard proofs of regret for episodic reinforcement learning algorithms (e.g. \cite{azar2017minimax,jaksch2010near}) but uses Proposition~\ref{prop:valbounds} to simplify the probabilistic analysis and the definition of the confidence sets to simplify the algebraic analysis.
For the proof, for any $h,t$, define $\Delta_{h,t}(x_{h,t}) = \val^+_{h,t}(x_{h,t}) - V^{\pi_t}_h(x_{h,t})$. %
Then using the optimistic result from Proposition~\ref{prop:valbounds}, on the event $\cap_{t=1}^K \{P \in \cP_t\}$, we can write the regret as
\begin{align*}
\Reg_T & = \sum_{t=1}^K (V^*_1(x_{1,t}) - V^{\pi_t}_1(x_{1,t})) \leq  \sum_{t=1}^K(\val^+_{1,t}(x_{1,t}) - V^{\pi_t}_{1,t}(x_{1,t})) = \sum_{t=1}^K \Delta_{1,t}(x_{1,t}).
\end{align*}
Then, for a fixed $h,t$, we consider $\Delta_{h,t}(x_{h,t})$ and show that this can be bounded in terms of $\Delta_{h+1,t}(x_{h+1,t})$, some confidence terms and some martingales. In particular, using the Bellman equations and the dynamic programming formulation, 
we can write
\begin{align}
&\Delta_{h,t}(x_{h,t}) = \val^+_{h,t}(x_{h,t}) - V^{\pi_t}_h(x_{h,t})  \nonumber
\\ &= \biprod{\wh{P}_{h,t}(\cdot|x_{h,t},a_{h,t}) }{\val^+_{h+1,t}} + r(x_{h,t},a_{h,t})+ \text{CB}_{h,t}(x_{h,t},a_{h,t}) - \biprod{P_h(\cdot|x_{h,t},a_{h,t})}{V^{\pi_t}_{h+1}} - r(x_{h,t},a_{h,t})  \nonumber
\\ & =  \biprod{\wh{P}_{h,t}(\cdot|x_{h,t},a_{h,t})}{\val^+_{h+1,t}}  - \biprod{P_h(\cdot|x_{h,t},a_{h,t}) }{V^{\pi_t}_{h+1}} + \text{CB}_{h,t}(x_{h,t},a_{h,t}) \nonumber
\\ & = \Delta_{h+1,t}(x_{h+1,t}) + \biprod{\wh{P}_{h,t}(\cdot|x_{h,t},a_{h,t}) }{\val_{h+1,t}^+} - \val_{h+1,t}^+(x_{h+1,t})   \nonumber
	\\ & \hspace{100pt} +V^{\pi_t}_{h+1}(x_{h+1,t}) -  \biprod{P_h(\cdot|x_{h,t},a_{h,t})}{V^{\pi_t}_{h+1}} + \text{CB}_{h,t}(x_{h,t},a_{h,t}) \nonumber
\\ & =  \Delta_{h+1,t}(x_{h+1,t}) +  \biprod{\wh{P}_{h,t}(\cdot|x_{h,t},a_{h,t}) - P_h(\cdot|x_{h,t},a_{h,t})}{\val_{h+1,t}^+} + \zeta_{h+1,t}^{\pi} + \text{CB}_{h,t}(x_{h,t},a_{h,t}) \nonumber
\end{align}
where in the last equality, $ \zeta_{h+1,t}^{\pi} $ is a martingale difference sequence defined by 
\[\zeta_{h+1,t}^{\pi} =  \biprod{P_h(\cdot|x_{h,t},a_{h,t})}{\val_{h+1,t}^+ - V^{\pi_t}_{h+1}} - \bpa{\val_{h+1,t}^+(x_{h+1,t})- V^{\pi_t}_{h+1}(x_{h+1,t})}.  \]

Then observe that on the event $P \in \cP_t$, 
\begin{align*}
 &\biprod{\wh{P}_{h,t}(\cdot|x_{h,t},a_{h,t}) - P_h(\cdot|x_{h,t},a_{h,t})}{\val_{h+1,t}^+}
 \\ & \leq \max_{\wt{P} \in \cP_h(x_{h,t},a_{h,t})} \biprod{\wh{P}_{h,t}(\cdot|x_{h,t},a_{h,t}) - \wt{P}_h(\cdot|x_{h,t},a_{h,t})}{\val_{h+1,t}^+}
 \\ &\leq \max_{\wt{P} \in \Delta} \bigg \{ \biprod{\wh{P}_{h,t}(\cdot|x_{h,t},a_{h,t}) - \wt{P}_h(\cdot|x_{h,t},a_{h,t})}{\val_{h+1,t}^+} :
 	\\ & \hspace{150pt}  D(\wt{P}_h(\cdot|x_{h,t},a_{h,t}),\wh{P}_{h,t}(\cdot|x_{h,t},a_{h,t})) \leq \epsilon_{h,t}(x_{h,t},a_{h,t}) \bigg\} 
 \\ & = \max_{\wt{P} \in \Delta} \bigg \{ \biprod{\wt{P}_h(\cdot|x_{h,t},a_{h,t})- \wh{P}_{h,t}(\cdot|x_{h,t},a_{h,t})}{-\val_{h+1,t}^+} : 
 	\\ & \hspace{150pt} D(\wt{P}_h(\cdot|x_{h,t},a_{h,t}),\wh{P}_{h,t}(\cdot|x_{h,t},a_{h,t})) \leq \epsilon_{h,t}(x_{h,t},a_{h,t}) \bigg\} 
 \\ &= D_*(- \val^+_{h+1,t}| \epsilon_{h,t}(x_{h,t},a_{ht}), \wh{P}_t(\cdot|x_{h,t},a_{h,t})) 
 \\ & = \CB_{h,t}^-(x_{h,t},a_{h,t})
 \end{align*}

This gives a recursive expression for $\Delta_{h,t}(x_{h,t})$,
\[ \Delta_{h,t}(x_{h,t})  \leq \Delta_{h+1,t}(x_{h+1,t})  + \zeta_{h+1,t}^{\pi} + \text{CB}_{h,t}(x_{h,t},a_{h,t}) +\text{CB}^-_{h,t}(x_{h,t},a_{h,t})\]
Recursing over $h=1,\dots, H$, we see that,
\begin{align*}
 \Delta_{1,t}(x_{1,t}) \leq \sum_{h=1}^H \text{CB}_{h,t}(x_{h,t}, \pi_t(x_{h,t})) +\sum_{h=1}^H \text{CB}^-_{h,t}(x_{h,t}, \pi_t(x_{h,t})) + \sum_{h=1}^H \zeta_{h+1,t}^{\pi}
\end{align*}
since $\Delta_{H+1,t}(x)=0$. 

By Azuma-Hoeffdings inequality, it follows that 
\[ \sum_{t=1}^K \sum_{h=1}^H \zeta_{h+1,t}^{\pi} \leq H \sqrt{2T \log(1/\delta)}  \]
with probability greater than $1-\delta$, since the sequence has increments bounded in $[-H,H]$.

Consequently, with probability greater than $1-\delta$, we can bound the regret by,
\begin{align*}
\Reg_T &%
 \leq \sum_{t=1}^K \sum_{h=1}^H \text{CB}_{h,t}(x_{h,t}, \pi_t(x_{h,t})) +  + \sum_{t=1}^K \sum_{h=1}^H \text{CB}_{h,t}^-(x_{h,t}, \pi_t(x_{h,t})) + H \sqrt{2T \log(1/\delta)}  
\end{align*}
thus giving the result.
\end{proof}

\subsection{Upper bounding the exploration bonus} \label{app:ubound}
We now prove the regret bound, when we use an upper bound $D_*^\dag$ on the conjugate $D_*$. 
We first need the below result that shows that the optimistic value function $\val^\dag$ in equation~\ref{eq:valbound} is indeed optimistic.

\begin{lemma} \label{lem:valbound}
On the event $P \in \cP$, it holds that $V^*_1(x_1) \leq \val^{\dag}_1(x_1)$. %
\end{lemma}
\begin{proof}
We consider the dual optimization problem,
\begin{align}
\underset{\val}{\text{minimize }}  \;& \val_1(x_1) \nonumber
\\ \text{subject to } \;& \val_h(x) \geq r(x,a) + \sum_y \hP_h(y|x,a) \val_{h+1}(y) + D_*\pa{\val_{h+1} \middle| \epsilon_h(x,a),\wh{P}_h(\cdot|x,a)}  \label{eq:val}
\\ & \val_h(x) \leq H-h+1 \label{eq:hbound}
\end{align}
which is the dual from Proposition~\ref{prop:dual}, where we have added the additional constraint \eqref{eq:hbound}. Note that adding this additional constraint will not effect the value of the optimal solution since by Proposition~\ref{prop:valbounds}, we know that $\val^+(x) \leq H-h+1$ for all $h=1,\dots, H, x \in\cS$. 

By definition of $D_*^\dag$, it follows that for any $\val_{h+1}$,
\begin{align*}
& r(x,a) + \sum_y \hP_h(y|x,a) \val_{h+1}(y) + D_*\pa{\val_{h+1} \middle| \epsilon_h(x,a),\wh{P}_h(\cdot|x,a)}  
 \\ & \hspace{20pt} \leq \min \bigg\{ H-h+1, r(x,a) + \sum_y \hP_h(y|x,a) \val_{h+1}(y) + D_*^\dag\pa{\val_{h+1} \middle| \epsilon'_h(x,a),\wh{P}_h(\cdot|x,a)}\bigg \} 
 \end{align*}
since all the original feasible solutions in stage $h+1$ must satisfy $\val_{h+1}(x) \leq H-h$.
Therefore, we can replace the constraint in \eqref{eq:val} by
\[ \val_h(x) \geq \min \bigg\{ H-h+1, r(x,a) + \sum_y \hP(y|x,a) \val_{h+1}(y) + D_*^\dag\pa{\val_{h+1} \middle| \epsilon'_h(x,a),\wh{P}(\cdot|x,a)} \bigg \}\]
knowing that this will only increase the optimal value of the objective function.
Since we know that by Proposition~\ref{prop:valbounds}, that the optimal solution to the original dual optimization problem satisfies $V^*_1(x_1) \leq \val^+_1(x_1)$ on the event $P \in \cP$, it must also be the case that $V^*_1(x_{1,t}) \leq \val^\dag_1(x_{1,t})$ for $\val^\dag_1(x_{1,t})$ the optimal solution of the modified dual. Note also that the solution to the modified dual problem will take the form given in \eqref{eq:valbound} by an argument similar to Lemma~\ref{lem:dualform}.
\end{proof}

\regub*
\begin{proof}
Given the result in Lemma~\ref{lem:valbound}, we know that $\val^{\dag}$ is optimistic so the proof proceeds similarly to the case where $\CB_{h,t}(x,a)$ is computed exactly. In  particular, let $\Delta_{h,t}^\dag(x_{h,t}) = \val^\dag(x_{h,t}) - V^\pi(x_{h,t})$, then,
\[ \Reg_T = \sum_{t=1}^K (V^*_1(x_{1,t}) - V^{\pi_t}_1(x_{1,t})) \leq  \sum_{t=1}^K(\val^\dag_{1,t}(x_{1,t}) - V^{\pi_t}_{1,t}(x_{1,t})) =\leq \sum_{t=1}^K \Delta_{h,t}^\dag(x_{h,t}) \]
and, observe that by the same argument as Theorem~\ref{thm:reggen},
\begin{align*}
\Delta_{h,t}^\dag(x_{h,t}) = \Delta_{h+1,t}^\dag(x_{h+1,t}) + \biprod{\wh{P}(\cdot|x_{h,t},a_{h,t})- P(\cdot|x_{h,t},a_{h,t})}{\val^\dag_{h+1,t}} + \zeta_{h+1,t}^\dag + \CB^\dag_{h,t}(x_{h,t},a_{h,t})
\end{align*}
where $\zeta_{h+1,t}^\dag$ is the martingale difference sequence $\zeta_{h+1,t}^\dag = \biprod{P(\cdot|x_{h,t},a_{h,t}}{\val_{h+1,t}^\dag - V^{\pi_t}_{h+1}} - (\val_{h+1,t}^\dag(x_{h+1,t})- V^{\pi_t}_{h+1}(x_{h+1,t})) $.
Then, on the event $P \in \cP$,
\begin{align*}
 & \hspace{-30pt} \biprod{\wh{P}_t(\cdot|x_{h,t},a_{h,t}) - P(\cdot|x_{h,t},a_{h,t})}{\val_{h+1,t}^\dag}
 \\ &\leq \max_{\wt{P} \in \Delta} \bigg \{ \biprod{\wh{P}_t(\cdot|x_{h,t},a_{h,t}) - \wt{P}}{\val_{h+1,t}^\dag} : D(\wt{P},\wh{P}_t(\cdot|x_{h,t},a_{h,t})) \leq \epsilon_{h,t}(x_{h,t},a_{h,t}) \bigg\} 
\\ & = \max_{\wt{P} \in \Delta} \bigg \{ \biprod{\wt{P}- \wh{P}_t(\cdot|x_{h,t},a_{h,t})} {-\val_{h+1,t}^\dag} : D(\wt{P},\wh{P}_t(\cdot|x_{h,t},a_{h,t})) \leq \epsilon(x_{h,t},a_{h,t}) \bigg\} 
 \\ &= D_*(- \val^\dag_{h+1,t}|\wh{P}_t(\cdot|x_{h,t},a_{h,t}), \epsilon_{h,t}(x_{h,t},a_{ht}))
 \\ &  \leq  D_*^\dag(\val^\dag_{h+1,t}|\wh{P}_t(\cdot|x_{h,t},a_{h,t}), \epsilon_{h,t}'(x_{h,t},a_{ht})) \leq  \CB_{h,t}^\dag(x_{h,t},a_{h,t})
 \end{align*}
 by definition of the upper bound $\CB_{h,t}^\dag(x,a)$.
 
 Using this, we can recurse over $h=1,\dots, H$ to get,
 \[ \Delta_{1,t}^\dag(x_{1,t}) \leq 2\sum_{h=1}^H \CB_{h,t}^\dag(x_{h,t}, a_{h,t}) + \sum_{h=1}^H \zeta_{h+1,t}^\dag \]
 so summing this over all episodes $t=1,\dots, K$ and using Azuma's inequality to bound the sum of the martingales gives the result.
\end{proof}

\subsection{Further Details of Examples} \label{app:examples}
Here we present additional results and explanations to show that many algorithms fit into our framework. The main 
purpose of this section is to demonstrate the use of our general results for constructing confidence sets and 
calculating the corresponding exploration bonuses, as well as bounding the regret. We do not aim to improve over 
state-of-the-art results or obtain tight constants, but we do note that several of the exploration bonuses we derive 
are data-dependent in a way that may possibly enable tight problem-dependent regret bounds. We refer to the works of 
\citet{dann2019policy,ZB19,simchowitz2019non} that demonstrate the power of data-dependent exploration bonuses for 
achieving such guarantees.

In several calculations below, we will use the following simple result to bound the sum of the exploration bonuses:
\begin{align} 
\sum_{t=1}^K \sum_{h=1}^H \sqrt{\frac{1}{N_{h,t}(x_{h,t},a_{h,t})}} &= \sum_{x\in \cS, a \in \cA}\sum_{t=1}^K \sum_{h=1}^H \one \{x_{h,t}=x,a_{h,t}=a\} \sqrt{\frac{1}{N_{h,t}(x_{h,t},a_{h,t})}}  \nonumber
\\ &= \sum_{x \in \cS,a \in \cA} \sum_{h=1}^H \sum_{n=1}^{N_{h,K}(x,a)} \sqrt{\frac{1}{n}} \leq  \sum_{x \in \cS,a \in \cA} \sum_{h=1}^H 2 \sqrt{N_{h,K}(x,a)} \nonumber
\\ &\leq 2 \sqrt{HSAT} \label{eq:pigeon}
\end{align}
where the last inequality follows due to the Cauchy--Schwarz inequality and the fact that $\sum_{x\in \cS,a \in \cA, h 
\in [H]} N_{h,K}(x,a) = HK = T$.
We also use the modified empirical transition probability defined for any states $x,x' \in \cS$, action $a \in \cA$, stage $h \in [H]$ and episode $t \in [K]$ as
\begin{equation} \label{eq:hpplus}
\hP^+_{h,t}(x'|x,a) = \frac{\max\{ 1, N_{h,t}(x,a,x')\}}{N_{h,t}(x,a)} 
\end{equation}
and note that this only differs from $\hP_{h,t}(x'|x,a)$ if $N_{h,t}(x,a,x')=0$. Consequently, 
\begin{equation}\label{eq:hpdiff}
|\hP^+_{h,t}(x'|x,a)  - \hP_{h,t}(x'|x,a) | = \bigg | \frac{\max\{ 1, N_{h,t}(x,a,x')\}}{N_{h,t}(x,a)} - \frac{N_{h,t}(x,a,x')}{N_{h,t}(x,a)} \bigg| \leq \frac{1}{N_{h,t}(x,a)}
\end{equation}
In several cases, we define the primal confidence sets using $\hP^+$ as the reference model rather than $\hP$ to avoid 
division by 0. 
Note that doing this results in dual formulations that involve $\hP^+$ rather than $\hP$. However, since we are still optimizing over the space of probability distributions in the primal, it holds that the optimal value of the dual objective will still be bounded by $H$.
We can also use Equation~\eqref{eq:hpdiff} to bound the empirical variance of any function $z:\cS \to [0,H]$ under 
$\hP^+$,
\begin{align} 
\wh{ \Var}^+(z) &= \sum_{y} \wh{P}^+(y) (z(y) - \siprod{\wh{P}^+}{z})^2  \leq \sum_{y} \wh{P}^+(y) \bigg( 2 (z(y) - \siprod{\wh{P}}{z})^2 +  2(\frac{HS}{N})^2\bigg)  \nonumber
\\ &\leq 2\sum_{y} \wh{P}(y)  (z(y) - \siprod{\wh{P}}{z})^2  + \frac{2HS+2(\frac{HS}{N})^2}{N}  + \frac{H^2S^2}{N^2} 
\leq 2 \wh{\Var}(z) + \frac{2HS}{N} + \frac{3H^2S^2}{N^2} \label{eq:varbound}.
\end{align}

\subsubsection{Total variation distance}  \label{sec:l1}
We start with the classic choice of the $\ell_1$ distance $D(p,p') = \onenorm{p-p'}$ which underlies the seminal UCRL2 algorithm of \citet{jaksch2010near}. Defining the confidence sets used in episode $t$ as 
\begin{align*}
 \cP_t &= \bigg\{ \wt{P} \in \Delta : \onenorm{\wt{P}_h(\cdot|x,a) - \wh{P}_t(\cdot|x,a)} \leq \epsilon_{h,t}(x,a) \quad \forall (x,a) \in \cZ, h \in [H] \bigg\}
\\ \text{ for } \quad \epsilon_{h,t}(x,a) &= \sqrt{\frac{2S \log(2SAT/\delta)}{N_{h,t}(x,a)}}
\end{align*}
we know that $P \in \cP_t$ for all $t=1,\dots, K$ with probability greater than $1-\delta$ \citep{jaksch2010near}.
Then, the conjugate distance is,
\begin{align*}
D_*(f|\epsilon,\wh{P}) &= \max_{P \in \Delta} \bigg \{ \biprod{P-\wh{P}}{f}\bigg| \onenorm{P-\wh{P}} \leq \epsilon \bigg\} = \min_{\lambda \in \bR} \max_{P \geq 0} \bigg \{ \biprod{P-\wh{P}}{f-\lambda {\bf 1}} \bigg| \onenorm{P-\wh{P}} \leq \epsilon \bigg \} 
\\ &\leq \min_{\lambda \in \bR} \max_{P \in \bR^S} \bigg \{ \biprod{P-\wh{P}}{f-\lambda {\bf 1} }\bigg| \onenorm{P-\wh{P}} \leq \epsilon \bigg \}   \leq \epsilon\min_{\lambda \in \bR} \|f-\lambda {\bf 1}\|_\infty \leq \epsilon \, sp(f)/2
\end{align*}
where we have defined $\lambda$ as the Lagrange multiplier of the constraint $\sum_x P(x) =1 = \sum_x \wh{P}(x)$, used the fact that the dual norm of the $\ell_1$ norm is the $\ell_\infty$ norm and, denoted by $sp(f)=\max_xf(x) - \min_xf(x)$ the span of $f$. Noting that a similar result holds for $D_*(-f|\epsilon,\wh{P})$, we can define $D_*^\dag(f|\epsilon,\wh{P}) = \epsilon \, sp(f)/2$, and use the exploration bonus
\[
 \CB_{h,t}^\dag(x,a) = \epsilon_{h,t}(x,a) sp(\val_{h+1,t}^\dag) /2%
\]
Since we are clipping $\val^+_h$ to be in the range $[0,H-h+1]$, we can bound $sp(\val_h^\dag) \leq H$.
Applying Theorem~\ref{thm:regub} and using the bound of Equation~\eqref{eq:pigeon} to bound the sum of the exploration bonuses 
shows that the regret of this algorithm is bounded by $\widetilde O(S\sqrt{AH^3T})$. %
This recovers the classic UCRL2 guarantees that can be deduced from the work of \cite{jaksch2010near}.

\subsubsection{Variance-weighted $\ell_\infty$ norm}  \label{sec:varl1}
We can get tighter bounds by using the empirical Bernstein inequality \cite{maurer2009empirical} to constrain the 
transition function. Here, we use $\hP^+ = \frac{\max \{ 1, N_{h,t}(x,a,y) \}}{N_{h,t}(x,a)}$ as the reference model in 
the primal confidence sets.
The constraints considered here are related to those used in the UCRL2B algorithm of \citet{improved_analysis_UCRL2B}. 
Specifically, we can apply the empirical Bernstein inequality to show that the following bound holds for all $x,a,x',h,t$ with probability  at least $1-\delta$:
\begin{align*}
\left|\hP_{h,t}^+(x'|x,a) - P_h(x'|x,a)\right|  & \leq \left|\hP_{h,t}(x'|x,a) - P_h(x'|x,a)\right|   + \left|\hP_{h,t}^+(x'|x,a) - \hP_{h,t}(x'|x,a)\right|  
\\& \hspace{-50pt} \le \sqrt{\frac{2\hP_{h,t}(x'|x,a)\pa{1- \hP_h(x'|x,a)}\log(HS^2AT/\delta)}{N_{h,t}(x,a)}} + \frac{7\log(HS^2AT/\delta)}{3N_{h,t}(x,a)} + \frac{1}{N_{h,t}(x,a)}
\\ & \hspace{-50pt}  \leq \sqrt{\frac{2\hP_{h,t}(x'|x,a)\log(HS^2AT/\delta)}{N_{h,t}(x,a)}} + \frac{7\log(HS^2AT/\delta)}{3N_{h,t}(x,a)}  + \frac{1}{N_{h,t}(x,a)}%
\\ & \hspace{-50pt}  \leq \sqrt{\frac{2\hP^+_{h,t}(x'|x,a)\log(HS^2AT/\delta)}{N_{h,t}(x,a)}} + \frac{7\log(HS^2AT/\delta)}{3N_{h,t}(x,a)}  + \frac{1}{N_{h,t}(x,a)} %
\\ & \hspace{-50pt}  \leq 6\log(HS^2AT/\delta)   \sqrt{\frac{\hP^+_{h,t}(x'|x,a)}{N_{h,t}(x,a)}}
\end{align*}  
The last inequality follows from the definition of the reference model that guarantees that %
 $N_{h,t}(x,a) \wh{P}^+_{h,t}(y|x,a) = \max\{N_{h,t}(x,a,y),1 \}\geq1$.

In what follows, we will state a confidence set inspired by the above result using the divergence measure $D(P,\hP^+) = \max_{x} \frac{(P(x) - \hP^+(x))^2}{\hP^+(x)}$, which is easily seen to be positive homogeneous and convex in both $P$ and $\hP^+$.
Defining $\epsilon_{h,t}(x,a) =   \frac{36\log^2(HS^2AT/\delta)}{N_{h,t}(x,a)}$,
we define the confidence sets used in episode $t$ as
\[ \cP_h(\cdot|x,a) = \left\{ P_h(\cdot|x,a) \in \Delta : \max_{y} \frac{(\wt{P}(y|x,a) - \wh{P}^+_{h,t}(y|x,a))^2}{\wh{P}^+_{h,t}(y|x,a)} \leq \epsilon_{h,t}(x,a) \right\}  \] %
and $ \cP = \cap_{x,a,h} \{\cP_h(\cdot|x,a)\}$. %
By the above argument, we know that $P \in \cP$ with probability greater than $1-\delta$.%

The corresponding conjugate distance can be expressed by defining $\lambda$ as the Lagrange multiplier of the constraint $\sum_x P(x) =1$ and writing
\begin{align*}
D_*(f|\epsilon,\wh{P}^+) &= \max_{P \in \Delta} \bigg\{ \biprod{P-\wh{P}^+}{f} \bigg|  \max_{x \in \cS} \frac{(P(x) - \hP^+(x))^2}{\hP^+(x)} \leq \epsilon \bigg \}
\\ & = \min_{\lambda \in \bR} \max_{P\geq 0}   \bigg\{ \biprod{P-\wh{P}^+}{f-\lambda{\bf 1}} - \lambda ( \sum_x \hP^+(x) -1) \bigg|  \max_{x \in \cS} \frac{|P(x) - \hP(x)|}{\sqrt{\hP^+(x)}} \leq \sqrt{\epsilon} \bigg \}
\\ & \leq  \min_{|\lambda| \leq H + \frac{SH}{N}} \max_{P \in \bR^S}   \bigg\{ \biprod{P-\wh{P}^+}{f-\lambda{\bf 1}}  - \lambda \sum_x (\hP^+(x) -\hP(x)) \bigg|  \max_{x \in \cS} \frac{|P(x) - \hP(x)|}{\sqrt{\hP^+(x)}} \leq \sqrt{\epsilon} \bigg \}
\\ & \leq \min_{|\lambda| \leq  H + \frac{SH}{N}} \sum_{x} \bigg| (f(x) -\lambda) \sqrt{\hP^+(x)}\bigg| \sqrt{\epsilon} + \bigg( H + \frac{SH}{N}\bigg) \frac{1}{N}
\\ &\leq \sqrt{\epsilon} \sum_{x} \sqrt{\wh{P}^+(x)} |f(x) - \wh{P}^+f| + \frac{2SH}{N}. %
\end{align*}
The same technique can be used to bound $D_*(-f|\epsilon,\wh{P}) $, so we can define $D_*^\dag(f|\epsilon,\wh{P})= \sqrt{\epsilon} \sum_x \sqrt{\wh{P}^+(x)} |f(x) - \wh{P}^+f| + \frac{2SH}{N}$ and write the inflated exploration bonus in the form
\[
 \CB_{h,t}^\dag(x,a) =\sqrt{ \epsilon_{h,t}(x,a)} \sum_{y} \sqrt{\wh{P}^+_{h,t}(y|x,a)} |\val^\dag_{h+1}(y) - \wh{P}_{h,t}^+\val^\dag_{h+1}| + \frac{2SH}{N_{h,t}(x,a)}. %
\]
By Theorem~\ref{thm:regub}, we know that in order to bound the regret of this algorithm, we need to be able to bound the sum of these exploration bonuses. For this, note that by the Cauchy--Schwarz inequality, and a similar argument to \eqref{eq:varbound},%
\begin{align*}
& \hspace{-30pt} \sqrt{ \epsilon_{h,t}(x,a)} \sum_{y} \sqrt{\wh{P}^+_{h,t}(y|x,a)} |\val^\dag_{h+1}(y)  - \wh{P}_{h,t}^+\val^\dag_{h+1}|  
\\ &\leq \sqrt{ \epsilon_{h,t}(x,a)} \sum_{y} \sqrt{\wh{P}_{h,t}(y|x,a)} |\val^\dag_{h+1}(y)  - \wh{P}_{h,t}\val^\dag_{h+1}|  + SH \sqrt{ \frac{\epsilon_{h,t}(x,a)}{N_{h,t}(x,a)}}  \bigg( 2 + \sqrt{\frac{H}{SN_{h,t}(x,a)}}\bigg) 
 \\  &=\sqrt{ \epsilon_{h,t}(x,a)} \sum_{y: P(y)>0} \sqrt{\wh{P}_{h,t}(y|x,a)} |\val^\dag_{h+1}(y) - \wh{P}_{h,t}\val^\dag_{h+1}| + 3SH  \sqrt{ \frac{\epsilon_{h,t}(x,a)}{N_{h,t}(x,a)}}  %
 \\ &\leq \sqrt{\epsilon_{h,t}(x,a)} \sqrt{ \Gamma_h(x,a)  \sum_{y: P(y)>0} \wh{P}^+_{h,t}(y|x,a) (\val^\dag_{h+1}(y) - \wh{P}_{h,t}\val^\dag_{h+1})^2} + 3SH  \sqrt{ \frac{\epsilon_{h,t}(x,a)}{N_{h,t}(x,a)}}%
 \\ &\leq \sqrt{\epsilon_{h,t}(x,a)\Gamma \wh{\Var}_{h,t}(\val^\dag_{h+1})} + 3SH \sqrt{\frac{ \epsilon_{h,t}(x,a)}{N_{h,t}(x,a)}}%
\end{align*}
where $\Gamma_h(x,a) $ is the number of next states which can be reached from state $x$ after playing action $a$ in stage $h$ with positive probability, and $\Gamma$ is a uniform upper bound on $\Gamma_h(x,a)$ that holds for all $x,a$, and $ \wh{\Var}_{h,t}$ is the empirical variance using all data from stage $h$ up to episode $t$. %
In order to bound $\CB_{h,t}^\dag(x_{h,t},a_{h,t}) \leq \sum_{t=1}^K \sum_{h=1}^H ( \sqrt{\epsilon_{h,t}(x_{h,t},a_{h,t}) \Gamma \wh{\Var}_{h,t}(\val^\dag_{h+1,t})} +  3SH \sqrt{\frac{ \epsilon_{h,t}(x,a)}{N_{h,t}(x,a)}})$, we use the Cauchy--Schwarz inequality and techniques similar to Lemma 10 in \cite{azar2017minimax} or Lemma 5 in \cite{improved_analysis_UCRL2B} to show that
\begin{align*}
&\sum_{t=1}^K \sum_{h=1}^H  \CB_{h,t}^\dag(x_{h,t},a_{h,t}) 
\\ &\leq C_1  \sqrt{\Gamma L} \sqrt{\sum_{t=1}^K \sum_{h=1}^H  \frac{1}{N_{h,t}(x_{h,t},a_{h,t})} \sum_{t=1}^K \sum_{h=1}^H \wh{\Var}_{h,t}(\val^\dag_{h+1,t})}  + C_4 SH \sqrt{L} \sum_{t=1}^K \sum_{h=1}^H \frac{1}{N_{h,t}(x_{h,t},a_{h,t})}   %
\\ &\leq C_1  \sqrt{\Gamma L} \sqrt{SA \log(T)  \bigg( \sum_{t=1}^K \sum_{h=1}^H \Var_h(V^{\pi_t}_{h+1})  +C_2 H^2 \sqrt{T\log(T)} \bigg) }  + C_4  SH \sqrt{L}  SA \log(T) %
\\ &  \leq C_1  \sqrt{\Gamma L} \sqrt{SA \log(T)  \bigg( HT +  C_3H^2\sqrt{T L} +C_2 H^2 \sqrt{T\log(T)} \bigg) } + C_4 SH \sqrt{L}  SA \log(T)  %
\\ & = \wt O( H\sqrt{\Gamma SAT})
\end{align*}
for some constants $C_1,C_2,C_3,C_4>0$, $L =\log(HS^2AT/\delta)$  and $\Var_h$ the variance under $P_h$, where the penultimate inequality follows from \cite{azar2017minimax} and the last inequality holds for $S^3 A \leq T\Gamma$.
This recovers the regret bounds of \citet{improved_analysis_UCRL2B}.

\subsubsection{Relative entropy} \label{app:kl}
Inspired by the KL-UCRL algorithm of \citet{filippi2010optimism}, we also consider the relative entropy (or 
Kullback--Leibler divergence, KL divergence) between $\wh{P}$ and $\wt{P}$ as a divergence measure. The relative 
entropy between two discrete probability distributions $p$ and $q$ is defined as 
\[
 D(p,q) = \sum_x p(x) \log \frac{p(x)}{q(x)},
\]
provided that $p(x)=0$ holds whenever $q(x)=0$. Being an $f$-divergence, the KL divergence satisfies the conditions 
necessary for our analysis: positive homogeneous and jointly convex in its arguments $(p,q)$. 
However, it is not symmetric in its arguments, which suggests that it can be used for defining confidence sets in two 
different ways, corresponding to the ordering of $P$ and $\hP$. We describe the confidence sets and the resulting exploration 
bonuses below.

\paragraph{Forward KL-Divergence.}
We first consider constraining the divergence $D(P,\hP) = \sum_y P(y) \log \bigg( \frac{P(y)}{\hP(y)} \bigg)$. 
To address the issue that the empirical transition probabilities $\hP(y)$ may be zero for some $y \in \cS$, 
we define the divergence with respect to $\hP^+$ (as defined in equation~\eqref{eq:hpplus}) and use the so-called 
\emph{unnormalized relative entropy} to account for the fact that $\hP^+$ may not be a valid probability distribution. 
Specifically, in what follows, we consider the following divergence measure:
\[ D(P,\hP) =  \sum_y P(y) \log \bigg( \frac{P(y)}{\hP^+(y)} \bigg) + \sum_y (\hP^+(y) - P(y)). \]
The following concentration result will be helpful for the construction of the confidence sets.
\begin{lemma}
With probability greater than $1-\delta$, it holds that for every episode $t$, stage $h$ and state-action pair $(x,a)$,
\[ D(P_{h}(x,a),\hP_{h,t}^+(x,a)) \leq \frac{18S \log(HSAT/\delta)}{N_{h,t}(x,a)} \] %
\end{lemma}
\begin{proof}
We consider a fixed $h,t,x,a$, and for ease of notation remove the dependence of $P,\hP$ on $h,t,x,a$. With probability greater than $1-\frac{\delta}{HTSA}$, it follows that
\begin{align*}
 \sum_y P(y) \log \bigg( \frac{P(y)}{\hP^+(y)} \bigg) &+ \sum_y (\hP^+(y) - P(y)) \leq \sum_y P(y) \bigg( \frac{P(y)}{\hP^+(y)}  - 1\bigg) + \sum_y (\hP^+(y) - P(y)) \tag*{(Since $\log(x) \leq x-1$ for $x>0$)}
 \\ & = \sum_y \frac{P^2(y) - P(y)\hP^+(y)}{\hP^+(y)} + \sum_y (\hP^+(y) - P(y))
 \\ & = \sum_y \frac{(P(y) - \hP^+(y))^2}{\hP^+(y)} %
 \\ & \leq 2 \sum_y \frac{(P(y) - \hP(y))^2}{\hP^+(y)}  + 2\sum_y \frac{(\hP(y) - \hP^+(y))^2}{\hP^+(y)} 
 \\ & \leq 2\sum_y \frac{2\hP(y) \log(HS^2AT/\delta)/N +6\log^2(HS^2AT/\delta)/N^2}{\hP^+(y)} + 2 \sum_y \frac{1}{N^2 \hP^+(y)}  \tag*{(By Bernstein's inequality and~\eqref{eq:hpdiff})} %
 \\ & \leq \frac{18S \log(HS^2AT/\delta)}{N} %
\end{align*}
where the last inequality follows since by definition $\hP^+(y)N \geq 1$.
Since this holds for each $h,t,x,a$ with probability greater than $1-\frac{\delta}{HTSA}$, by the union bound, it follows that it holds simultaneously for all $h,t,x,a$ with probability greater than $1-\delta$.
\end{proof}
Given the above result, we define our confidence set as
\[
\cP_{h,t}(\cdot |x,a) = \bigg \{ \wt{P}_h(\cdot|x,a) \in \Delta \bigg| \sum_{x'} \wt{P}_h(x'|x,a) \log\frac{\wt{P}_h(x'|x,a)}{\wh{P}_{h,t}^+(x'|x,a)} \le \epsilon_{h,t}(x,a) \bigg \}
\]
\[
\text{ for } \qquad  \epsilon_{h,t}(x,a) =\frac{C S\log(HSAT/\delta)}{N_{h,t}(x,a)}
\]
for some constant $C>0$. 
Using the notation $\text{KL}(p,q) = \sum_y p(y) \log(p(y)/q(y))$ to denote the normalized KL divergence, the conjugate 
of the above divergence can be written as
\begin{align*}
D_*(z|\epsilon,\wh{P}^+) 
& = \max_{\wt{P} \in \Delta} \ev{\iprod{z}{\wt{P} - \wh{P}^+}\middle| D(\wt{P},\wh{P}^+) \le \epsilon} 
 \\&= \min_{\lambda\ge 0} \max_{\tP\in \Delta} \ev{\iprod{z}{\wt{P} - \wh{P}^+} - \lambda\pa{D(\wt{P},\wh{P}^+) - \epsilon}} %
 \\ & =  \min_{\lambda\ge 0} \max_{\tP\in \Delta}  \ev{\iprod{z}{\wt{P} - \wh{P}^+} - 
\lambda\pa{\text{KL}(\wt{P},\wh{P}^+) 
 + \iprod{\bm{1}}{\hP^+ - \wt{P}} - \epsilon}} %
 \\ & =  \min_{\lambda\ge 0} \max_{\tP\in \Delta}  \ev{\iprod{z}{\wt{P} - \wh{P}^+} - 
\lambda\pa{\text{KL}(\wt{P},\wh{P}^+)
 - \epsilon'}} %
\\ 
&= \min_{\lambda\ge 0}\bigg\{ \lambda \log \sum_{x'} \wh{P}^+(x') e^{z(x)/\lambda} - \sum_{x'} \wh{P}^+(x') z(x') + 
\lambda \epsilon'\bigg\} %
\end{align*}
where we defined $\epsilon' = \epsilon + 1 - \siprod{\bm{1}}{\hP^+}$ and used the well-known Donsker--Varadhan 
variatonal formula (see, e.g., \citep[Corollary~4.15]{BLM13}) in the last line. Thus, the exploration bonus can be 
efficiently calculated by a line-search procedure to find the $\lambda$ minimizing the expression above.

A more tractable bound on the exploration bonus can be provided by noting that, for a vector $z$ with $\infnorm{z}\le 
H$, we have
\begin{align*}
D_*(z|\epsilon,\wh{P}^+) 
&= \min_{\lambda\ge 0}  \bigg \{ \lambda \log \sum_{y} \wh{P}^+(y) e^{z(y)/\lambda} - \sum_{y} \wh{P}^+(y) z(y) + \lambda 
\epsilon' \bigg\} 
\\ &= \min_{\lambda \geq 0} \bigg\{ \lambda \log \sum_{y}\wh{P}^+(y) e^{(z(y) - \siprod{\wh{P}^+}{z})/\lambda}+ \lambda 
\epsilon' \bigg\}
\\ & \le \min_{\lambda \in [0,H]} \bigg\{ \lambda \log \sum_{y}\wh{P}^+(y) e^{(z(y) - \siprod{\wh{P}^+}{z})/\lambda}+ 
\lambda \epsilon' \bigg\}
\\
&\le \min_{\lambda \in [0,H]} \bigg \{ \frac{1}{\lambda} \sum_{y} \wh{P}^+(y) \pa{z(y) - \siprod{\wh{P}^+}{z}}^2 + \lambda 
\epsilon' \bigg \}
 \\ &\le 2\sqrt{\epsilon' \sum_{y} \wh{P}^+(y) (z(y) - \siprod{\wh{P}^+}{z})^2} %
 = 2\sqrt{\epsilon' \wh{\Var}^+(z)} %
\end{align*}
where we used the inequality $\lambda \log \E^+[e^{X/\lambda}] \le \E^+[X] + \frac{1}{\lambda} \E^+[X^2]$ for $\E^+[X] = 
\sum_x \hP^+(x)x$  that holds 
as long as $|X|\le \lambda$ holds almost surely, and the result in Equation~\eqref{eq:hpdiff} several times. We also 
use the notation $\wh \Var^+(z)$ to denote the variance of $z$ under $\wh{P}^+$.
Thus, defining 
\[
\epsilon'_{h,t}(x,a) =  \epsilon_{h,t}(x,a) + \sum_y \wh{P}^+(y|x,a) - 1 \leq \epsilon_{h,t}(x,a) + \frac{S - 
\Gamma}{N_{t,h}(x,a)} = \wt O \bigg(\frac{S}{N_{h,t}(x,a)} \bigg),
\]
the exploration bonus can be bounded as $\CB_{h,t}(x,a) \leq 2\sqrt{\epsilon'_{h,t}(x,a) 
\wh{\Var}^+_{h,t}(\val^+_{h+1,t})} $, %
and using an identical argument yields the same bound for $\CB^-_{h,t}(x,a)$.

By \eqref{eq:hpdiff}, $\wh{ \Var}^+(z) \leq  2 \wh{\Var}(z) + \frac{2HS}{N} + \frac{3H^2S^2}{N^2}$
and so the exploration bonus can be bounded in the same way as in the case of variance-weighted $\ell_\infty$ constraints, plus some lower order terms that scale with $1/N$. %
The sum of these lower order terms can be straightforwardly bounded by a simple adaptation of the calculations in Equation~\eqref{eq:pigeon}. %
Overall, the sum of the confidence bounds can be bounded as
\[ \sum_{t=1}^K \sum_{h=1}^H (\CB_{h,t}(x_{h,t},a_{h,t}) + \CB_{h,t}^-(x_{h,t},a_{h,t})) \leq C_1 HS \sqrt{AT}  + C_2 H^2S^2 A\log T \]
for some $C_1,C_2 = O(\log(HSAT/\delta))$.
Hence the regret can be bounded by $\wt{O}(HS\sqrt{AT})$.

\paragraph{Reverse KL-Divergence.}
We now consider defining confidence sets in terms of the second argument of the KL divergence, corresponding the the 
original KL-UCRL algorithm proposed by \citet{filippi2010optimism,talebi2018variance}. Specifically, define,
\begin{align*}
\cP_{h,t}(\cdot |x,a) &= \bigg \{ \wt{P}_h(\cdot|x,a) \in \Delta \bigg| \sum_{x'} \wh{P}_h(x'|x,a) \log\frac{\wh{P}_h(x'|x,a)}{\wt{P}_{h,t}(x'|x,a)} \le \epsilon_{h,t}(x,a) \bigg \}
\\ \text{ for } \qquad  \epsilon_{h,t}(x,a) &=\frac{CS\log(HSAT/\delta)}{N_{h,t}(x,a)}.
\end{align*}
for some constant $C>0$. 
As shown by \citet{filippi2010optimism}, for an appropriate choice of $C$, this confidence set is guaranteed to capture the true transition function in all episodes with probability greater than $1-\delta$.

The conjugate of this distance for a fixed $x,a$ can be bounded as
\begin{align*}
D_*(z|\epsilon,\wh{P}) 
& = \max_{\wt{P} \in \Delta} \ev{\iprod{z}{\wt{P} - \wh{P}}\middle| D(\wt{P},\wh{P}) \le \epsilon} 
\\ & = \min_{\lambda \geq 0} \max_{\wt{P} \in \Delta}  \bigg\{ \iprod{z}{\wt{P} - \wh{P}} - \lambda(D(\wt{P},\wh{P}) - \epsilon) \bigg\}
\\ & \leq \min_{\lambda \geq 0} \max_{\wt{P} \in \Delta} \bigg\{ \iprod{z}{\wt{P} - \wh{P}} - \lambda(1/2 \|\wt{P} - \wh{P}\|_1^2- \epsilon) \bigg\} \tag*{(By Pinsker's inequality)}
\\ & \leq sp(z)\sqrt{2\epsilon }
\end{align*}
where the last inequality follows by an argument similar to the results for the total variation distance in 
Section~\ref{sec:l1} using the fact that the dual of the $\ell_1$ norm is the $\ell_\infty$ norm. 

Similarly, it can be shown that $D_*(-z|\epsilon,\wh{P})  \leq  sp(z) \sqrt{2\epsilon}$. Therefore, we define the confidence bounds,
\[ \CB_{h,t}^{\dag}(x,a) = sp(\val_{h+1,t}^\dag) \sqrt{2\epsilon_{h,t}(x,a)}. \]
By Theorem~\ref{thm:regub}, we know the regret can be bounded in terms of the sum of these confidence bounds. Consequently, using equation~\ref{eq:pigeon}, we see that,
\[ \sum_{t=1}^K \sum_{h=1}^H \CB_{h,t}^{\dag}(x_{h,t},a_{h,t}) \leq  HS\sqrt{2HAT} \log(HSAT/\delta). \]
Hence the regret can be bounded by $\wt O(S\sqrt{H^3AT})$.
This matches the regret bound in \citet{filippi2010optimism}. %
Using an alternative analysis essentially corresponding to a tighter bound on the conjugate distance, 
\citet{talebi2018variance} were able to prove a regret bound of $\wt{O}(\sqrt{S \sum_{h,x,a} \Var_{h-1}(V^*_h(x,a))T})$ 
for KL-UCRL where $ \Var_{h-1}(V^*_h(x,a))$ is the variance of $V^*_{h}$ after playing action $a$ from state $s$ in 
stage $h-1$. 
We conjecture that it is possible to obtain a regret bound of $\wt{O}(H\sqrt{\Gamma SAT})$ by combining the techniques of \citet{talebi2018variance} and \citet{azar2017minimax}.

\subsubsection{$\chi^2$-divergence}
We can also use the Pearson $\chi^2$-divergence to define the primal confidence sets in~\eqref{eq:primalMDPopt}.  Specifically, we consider the distance
\[D(P, \hP^+) = \sum_y \frac{(P(y) - \hP^+(y))^2}{\hP^+(y)}, \]
for $\hP^+$ defined as in equation~\eqref{eq:hpplus}
and note that similar results hold for the distance $D(P,\hP) = \sum_{y} \frac{ P^2(y)-\hP^2(y)}{\hP(y)}$.
We will use $\hP^+$ as the reference model for the primal confidence sets.
Using the empirical Bernstein inequality \cite{maurer2009empirical}, we see that with probability greater than $1-\delta$, for all episodes $t$, $a\in \cA, x \in \cS, h \in [H]$,
\begin{align*}
&D(P_h(\cdot|x,a), \wh{P}_{h,t}^+(\cdot|x,a)) = \sum_{y} \frac{(P_h(y|x,a) - \wh{P}_{h,t}^+(y|x,a))^2}{\wh{P}^+_{h,t}(y|x,a)} 
\\ & \qquad \qquad \leq  2\sum_{y} \frac{(P_h(y|x,a) - \wh{P}_{h,t}(y|x,a))^2}{\wh{P}^+_{h,t}(y|x,a)}  + 2 \sum_{y} \frac{(\hP_h(y|x,a) - \wh{P}^+_{h,t}(y|x,a))^2}{\wh{P}^+_{h,t}(y|x,a)} 
\\ &\qquad\qquad \leq \sum_{y} \bigg(\frac{2 \wh{P}_{h,t}(y|x,a)(1-\wh{P}_{h,t}(y|x,a))\log(HS^2AT/\delta)}{N_{h,t}(x,a) \wh{P}^+_{h,t}(y|x,a)}  
    + \frac{49 \log^2(HS^2AT/\delta)}{9N_{h,t}^2(x,a) \wh{P}_{h,t}^+(y|x,a)} \bigg) + \frac{2S}{N_{h,t}(x,a)}
\\ &\qquad\qquad \leq \sum_{y} \bigg(\frac{2 \wh{P}^+_{h,t}(y|x,a)\log(HS^2AT/\delta)}{N_{h,t}(x,a) \wh{P}^+_{h,t}(y|x,a)} 
    + \frac{49 \log^2(HS^2AT/\delta)}{9N_{h,t}(x,a)} \bigg)  + \frac{2S}{N_{h,t}(x,a)}
\\ &\qquad\qquad  \leq \frac{11 S \log^2(HS^2AT/\delta)}{N_{h,t}(x,a)} 
\end{align*}
where the second to last inequality follows since %
$N_{h,t}(x,a) \wh{P}^+_{h,t}(y|x,a) = \max\{1,N_{h,t}(x,a,y) \}\geq1$, and $\hP_{h,t}^+(y|x,a) \geq \hP_{h,t}(y|x,a)$.
We can then define the confidence sets as
\[ \cP_{h,t}(\cdot|x,a) = \bigg\{ \wt{P}_h \in \Delta \bigg| D(\wt{P}_h(\cdot|x,a), \wh{P}_{h,t}^+(\cdot|x,a)) \leq \epsilon_{h,t}(x,a)  \bigg \} \, \text{ for } \, \epsilon_{h,t}(x,a) =  \frac{11 S \log^2(HS^2AT/\delta)}{N_{h,t}(x,a)}. \] 
Furthermore, the conjugate $D_*(\val|\epsilon,\wh{P}^+)$ can be written as follows:
\begin{align*}
D_*(\val|\epsilon,\wh{P}^+) &= \max_{P \in \Delta} \bigg\{ \biprod{P-\wh{P}^+}{\val} : D(P,\wh{P}^+) \leq \epsilon \bigg\} 
\\ &= \min_{\lambda \in \bR} \max_{P\geq 0} \bigg\{ \biprod{P-\wh{P}^+}{\val-\lambda{\bf 1}} - \lambda (\sum_y \hP^+(y) - 1) : \bigg\|\frac{P-\hP^+}{\sqrt{\hP^+}}\bigg\|_2^2 \leq \epsilon \bigg\}
\\ & =\min_{\lambda \in \bR} \max_{P} \bigg \{ \biprod{P-\wh{P}^+}{\val-\lambda{\bf 1}} - \lambda \sum_y (\hP^+(y) - \hP(y))  : \bigg \|\frac{P-\hP^+}{\sqrt{\hP^+}} \bigg\|_2 \leq \sqrt{\epsilon} \bigg\}
\\ &=  \min_{\lambda \leq H  +\frac{SH}{N}}  \sqrt{\epsilon  \sum_{y} \wh{P}^+(y)(\val(y) -\lambda)^2 }  + \bigg(H+\frac{SH}{N}\bigg) \frac{1}{N}
\\ & \leq \sqrt{\epsilon \wh{\Var}^+(\val)} + \frac{2SH}{N}
\end{align*}
where we have used properties of the dual of the weighted $\ell_2$ norm. %
Therefore, both $\CB_{h,t}(x,a)$ and $\CB^-_{h,t}(x,a)$ can be upper-bounded by 
for $\CB_{h,t}^\dag(x,a)= \sqrt{\epsilon_{h,t}(x,a) \wh{\Var}^+_{h,t}(\val^+_{h+1,t})} $ %
and we can apply Theorem~\ref{thm:regub} to show that the regret is bounded by the sum of these exploration bonuses. Following the same steps as in Section~\ref{sec:varl1} and using the bound on the variance under $\hP^+$ in \eqref{eq:varbound}, this eventually leads to a regret bound of $\wt{O}(HS\sqrt{ AT})$. %

It is interesting to note that \citet{maillard2014hard} considered similar confidence sets using a \emph{reverse} $\chi^2$-divergence defined 
as $D(p,q) = \sum_{y} \frac{ q^2(y)-p^2(y)}{p(y)}$. Using this distance with a feasible confidence set would fit into 
our framework. However, for their regret analysis, \citet{maillard2014hard} impose the additional constraint that for 
all $x'$ such that $\wt{P}_{h,t}(x'|x,a)>0$, it must also hold that $\wt{P}_{h,t}(x'|x,a)>p_0$ for some positive $p_0$. 
Unfortunately, this constraint makes the set $\cP$ non-convex\footnote{To see this, consider $\tilde p$ and $\tilde p'$ 
satisfying the constraints, which differ only in $x$ where $\tilde p(x)=p_0$ and $\tilde p'(x) =0$. Then, nontrivial 
convex combinations of $\tilde p, \tilde p'$ no longer satisfy the constraints.} and thus their eventual approach does 
not entirely fit into our framework. Finally, we note that the bounds of \citet{maillard2014hard} replace a factor of 
$S$ appearing in our bounds by $1/p_0$, which may in an inferior bound when $p_0$ is small. Overall, we believe that 
the Pearson $\chi^2$-divergence we propose in this section can remove this limitation of the analysis of
\citet{maillard2014hard} while also retaining the strong problem-dependent character of their bounds.

\section{Results for Linear Function Approximation}\label{app:linproofs}
In this section, we provide proofs of the results in the linear function approximation setting.  
Throughout the analysis, we will use the notation 
\[
C_t(\delta) = 2 H\sqrt{d \log\pa{1 + t R^2/\lambda} + \log(1/\delta)} + C_P H \sqrt{\lambda d} 
\]
where $C_P$ is such that $\|m_{h,a}(x)\|_1 \leq C_P$ for every row $m_{h,a}(x)$ of $M_{h,a}$ and $R$ is such that $\|\varphi(x)\|_2\leq R$ for all $x \in \cS$.
We also define the event 
\[
 \mathcal{E}_{h,a,t}(g,\delta) = \ev{\norm{\pa{M_{h,a} - \hM_{h,a,t}}g}_{\Sigma_{h,a,t-1}} \le C_t(\delta)}.
\]

We start by proving our key concentration result that will be used for deriving our confidence sets.
\begin{restatable}{proposition}{conc_lin}\label{lem:conc_lin}
Consider the reference model $\wh{P}_{h,a,t} = \Phi \hM_{h,a,t}$ with $\hM_{h,a,t}$ defined in Equation~\eqref{eq:Mhat}.
Then, for any $a \in \cA, h \in [H]$, episode $t$ and any fixed function $g:\Sw\ra [-H,H]$, the following holds with probability at least $1-\delta$:
\[
\norm{\pa{M_{h,a} - \hM_{h,a,t}}g}_{\Sigma_{h,a,t-1}} \le 2 H\sqrt{d \log\pa{1 + t R^2/\lambda} + \log(1/\delta)} + C_P H \sqrt{\lambda d}.
\] 
\end{restatable}
\begin{proof}
We start by rewriting 
\[
 \norm{\pa{M_{h,a} - \hM_{h,a,t}}g}_{\Sigma_{h,a,t-1}} = \norm{\Sigma_{h,a,t-1} \pa{M_{h,a} - \hM_{h,a,t-1}}g}_{\Sigma_{h,a,t-1}^{-1}},
\]
and proceed by using the definitions of $\hM_{h,a,t}$, $\Sigma_{h,a,t-1}$ and $W_{h,a,t-1}$ to see that
\begin{align*}
 \Sigma_{h,a,t-1} \bpa{M_{h,a} - \hM_{h,a,t}}g  
 &= \Phi\transpose W_{h,a,t-1} \Phi M_{h,a} g + \lambda M_{h,a} g 
 	 \\ & \hspace{50pt} -  \Sigma_{h,a,t-1} \Sigma_{h,a,t-1}^{-1} \sum_{k=1}^{t-1} \II{a_{h,k} = a} \varphi(x_{h,k}) g\pa{x_{h+1,k}}  
\\
&= \Phi\transpose W_{h,a,t-1} P_{h,a} g - \sum_{k=1}^{t-1} \II{a_{h,k} = a} \varphi(x_{h,k}) g\pa{x_{h+1,k}}  + \lambda M_{h,a} g 
\\
&=\sum_{k=1}^{t-1} \II{a_{h,k} = a} \pa{\biprod{P_{h}(\cdot|x_{h,k},a_{h,k})}{g} - g(x_{h+1,k})}\varphi (x_{h,k}) + \lambda M_{h,a} g. 
\end{align*}
The first term on the right-hand side is a vector-valued martingale for an appropriately chosen filtration, since 
\[
\EEcc{\biprod{P_{h}(\cdot|x_{h,k},a_{h,k})}{g} - g(x_{h+1,k})}{x_{h,k},a_{h,k}} = 0,
\]
so the sum of these terms can be bounded by appealing to 
Theorem~1 of \citet{abbasi2011improved} as
\begin{align*}
 &\norm{\sum_{k=1}^{t-1} \II{a_{h,k} = a} \pa{\biprod{P_{h}(\cdot|x_{h,k},a_{h,k})}{g} - g(x_{h+1,k})}\varphi\transpose (x_{h,k})}_{\Sigma_{h,a,t-1}^{-1}} 
 \\
 &\qquad \qquad \qquad \qquad \qquad \qquad \qquad \qquad \qquad \qquad \le 2 H\sqrt{d \log\pa{1 + t R^2/\lambda} + \log(1/\delta)}.
\end{align*}
The proof is concluded by applying the bound
\[
 \norm{\lambda M_{h,a} g}_{\Sigma_{h,a,t-1}^{-1}} \le \sqrt{\lambda} \norm{M_{h,a} g} \le C_P H \sqrt{\lambda d},
\]
where in the last step we used the assumption that $\onenorm{m_{h,a}(x)}\le C_P$ and $\infnorm{g}\le H$.
\end{proof}
The following simple result will also be useful in bounding the sum of exploration bonuses and thus the regret of the two algorithms:
\begin{lemma}\label{lem:sumbound}
For any $h \in [H]$,
 \[
  \sum_{a\in \cA} \sum_{t=1}^K \norm{\II{a_{h,t} = a}\varphi(x_{h,t})}_{\Sigma_{h,a,t-1}^{-1}} \le 2 \sqrt{d A K\log\pa{1 + K R^2/\lambda}}.
 \]
\end{lemma}
\begin{proof}
 The claim is directly proved by the following simple calculations:
 \begin{align*}
 \sum_{a \in \cA} \sum_{t=1}^K \norm{\II{a_{h,t} = a}\varphi(x_{h,t})}_{\Sigma_{h,a,t-1}^{-1}}
 &\le \sqrt{\sum_{a}\sum_{t=1}^K \II{a_{h,t} = a}} \sqrt{\sum_a \sum_{t=1}^K \norm{\II{a_{h,t} = a}\varphi(x_{h,t})}_{\Sigma_{h,a,t}^{-1}}^2}
 \\
 & \hspace{-20pt} \le 2 \sqrt{K \sum_a \log\pa{\frac{\textup{det}\pa{\Sigma_{h,a,K}}}{\textup{det}\pa{\lambda I}}}} 
 \le 2 \sqrt{K d A\log\pa{1 + K R^2/\lambda}},
\end{align*}
where the first inequality is Cauchy--Schwarz and the second one follows from Lemma~11 of \citet{abbasi2011improved}.
\end{proof}
Finally, the following result will be useful to bound the scale of the esimated model $\hM_{h,a,t}$ with probability $1$:
\begin{lemma}\label{lem:mhatscale}
 Consider the reference model $\wh{P}_{h,a,t} = \Phi \hM_{h,a,t}$ with $\hM_{h,a,t}$ defined in Equation~\eqref{eq:Mhat}.
Then, for any $B>0$ and any fixed function $g:\Sw\ra [-B,B]$, the following statements hold with probability $1$:
\[
\norm{\hM_{h,a,t} g} \le \frac{tBR}{\lambda} \qquad \mbox{and} \qquad \norm{\pa{M_{h,a} - \hM_{h,a,t}}g}_{\Sigma_{h,a,t-1}} \le \lambda^{-1/2} tBR + \lambda^{1/2} BC_P.
\]
\end{lemma}
\begin{proof}
The first statement is proven by straightforward calculations, using the definition of $\hM_{h,a,t}$:
\begin{align*}
 \norm{\hM_{h,a,t} g} &= \norm{\Sigma_{h,a,t-1}^{-1} \sum_{k=1}^{t-1} \II{a_{h,k} = a} \varphi(x_{h,k}) g\pa{x_{h+1,k}}}
 \\&\le \norm{\Sigma_{h,a,t-1}^{-1}}_{\text{op}} \norm{\sum_{k=1}^{t-1} \II{a_{h,k} = a} \varphi(x_{h,k}) g\pa{x_{h+1,k}}}
 \le \frac{B}{\lambda} \sum_{k=1}^{t-1} \norm{\varphi(x_{h,k})} \le \frac{tBR}{\lambda},
\end{align*}
where the second inequality uses that the operator norm of $\Sigma_{h,a,t-1}^{-1}$ is at most $\lambda^{-1}$, and the triangle inequality.
As for the second inequality, we proceed as in the proof of Proposition~\ref{lem:conc_lin} and recall that
\begin{align*}
 \Sigma_{h,a,t-1} \pa{M_{h,a} - \hM_{h,a,t}}g  
&=\sum_{k=1}^{t-1} \II{a_{h,k} = a} \pa{\biprod{P_{h}(\cdot|x_{h,k},a_{h,k})}{g} - g(x_{h+1,k})}\varphi (x_{h,k}) + \lambda M_{h,a} g. 
\end{align*}
The norm of the above is clearly bounded by $tBR + \lambda BC_P$. Thus, we have
\begin{align*}
 \norm{\pa{M_{h,a} - \hM_{h,a,t}}g}_{\Sigma_{h,a,t-1}} &= \norm{\Sigma_{h,a,t-1} \pa{M_{h,a} - \hM_{h,a,t}g}}_{\Sigma_{h,a,t-1}^{-1}}
 \\ & \le \norm{\Sigma_{h,a,t-1}^{-1/2}}_{\text{op}} 
 \norm{\Sigma_{h,a,t-1} \pa{M_{h,a} - \hM_{h,a,t}}g} 
 \\
 &\le \frac{1}{\sqrt{\lambda}} \pa{tBR + \lambda BC_P} = \lambda^{-1/2} tBR + \lambda^{1/2} BC_P.
\end{align*}
This concludes the proof.
\end{proof}

\subsection{Optimism in state space through local confidence sets}
This section presents our approach for factored linear MDPs with local confidence sets, which can be seen to lead to confidence bonuses 
in the state space. We first state some structural results that will justify our algorithmic approach, explain our algorithm in more detail, 
and then present the performance guarantees.

We recall that our approach is based on solving the following optimization problem:
\begin{align*}
\underset{q \in \cQ(x_1), \omega, \wt P}{\text{maximize}} \quad & \sum_{h=1}^H \sum_a\iprod{ W_{h,a,t-1}\Phi 
\omega_{h,a}}{r_a} &
 \\
   \text{subject to} \quad & \sum_a {q}_{h+1,a} = \sum_a \wt{P}_{h,a} W_{h,a,t-1} \Phi \omega_{h,a} \qquad 
   &\forall a \in \cA, h=1,\dots, H
\\  & \Phi\transpose {q}_{h,a} = \Phi\transpose W_{h,a,t-1} \Phi \omega_{h,a} \qquad &\forall a \in 
\cA, 
h=1,\dots, H
\\& D\pa{\wt{P}_h(\cdot|x,a), \wh{P}_{h,t}(\cdot|x,a)} \leq \epsilon_{h,t}(x,a) \qquad &\forall (x,a),
\end{align*}
where $D$ is an arbitrary divergence that is positive homogeneous and convex in its arguments.
The following structural result shows that this optimization problem can be equivalently written in a dual form
that is essentially identical to the optimistic Bellman equations derived in Section~\ref{sec:tab} for the tabular setting.
\begin{restatable}{proposition}{dual_approx_1}\label{prop:dual_approx_1}
The optimization problem above is equivalent to solving the optimistic Bellman equations~\eqref{eq:bellman_opt_lin} 
with the exploration bonus defined as
\[
 \CB_{h}(x,a) = D^*\pa{\val^+_{h+1} \middle| \epsilon_h(x,a),\wh{P}_h(\cdot|x,a)}.
\]
\end{restatable}
The proof follows from a similar reparametrization as used in the proof of Proposition~\ref{prop:dual} 
that makes the optimization problem convex, thus enabling us to establish strong duality. To maintain readability, we defer 
the proof to Appendix~\ref{app:dual_approx_1}. Consequently, the properties stated in Propositions~\ref{prop:pols} and~\ref{prop:valbounds} can also 
be shown in a straightforward fashion.

Our results are based on using the divergence measure 
\[
 D\pa{\wt{P}_{h,t}(\cdot|x,a), \wh{P}_{h,t}(\cdot|x,a)} = \sup_{g\in\V_{h+1,t}} \iprod{\tP_{h,t}(\cdot|x,a) - \hP_{h,t}(\cdot|x,a)}{g}, %
\]
whose conjugate can be directly upper-bounded by $\epsilon_{h,t}$. Since the structural results established above directly imply that Theorem~\ref{thm:reggen} continues to hold, we can easily derive a practical and effective algorithm by simply using $\epsilon_{h,t}$ as the exploration bonuses. 
Specifically, we will consider an algorithm that calculates an optimistic value function and a corresponding policy by solving the OPB~equations~\eqref{eq:bellman_opt_lin} via dynamic programming, with the confidence bonuses chosen as
\[
 \CB_{h,t}^\dag(x,a) = \alpha_{h,t} \norm{\varphi(x)}_{\Sigma_{h,a,t-1}^{-1}}
\]
for some $\alpha_{h,t}$. The shape of this confidence set is directly motivated by the following simple corollary of our general concentration result in Lemma~\ref{lem:conc_lin}:
\begin{lemma} \label{lem:conc_lin_local}
Fix $h,a$ and consider the reference model $\wh{P}_{h,a,t} = \Phi \hM_{h,a,t}$ with $\hM_{h,a,t}$ defined in Equation~\eqref{eq:Mhat}.
Then, for any fixed function $g:\Sw\ra [-H,H]$, the following holds simultaneously for all $x$ under event $\mathcal{E}_{h,a,t}(g,\delta)$:
\[
 \iprod{P_{h}(\cdot|x,a) - \hP_{h,t}(\cdot|x,a)}{g} \le C_t(\delta) \norm{\varphi(x)}_{\Sigma_{h,a,t-1}^{-1}}.
\]
\end{lemma}
\begin{proof}
The proof is immediate using the definition of the event $\mathcal{E}_{h,a,t}(g,\delta)$ and the Cauchy--Schwarz inequality:
 \begin{align*}
  \iprod{P_h(\cdot|x,a) - \hP_{h,t}(\cdot|x,a)}{g} &= \biprod{\varphi(x)}{\bpa{M_{h,a} - \hM_{h,a,t}} g}
 \\ & \hspace{-20pt} \le \norm{\varphi(x)}_{\Sigma_{h,a,t-1}^{-1}} \norm{\pa{M_{h,a} - \hM_{h,a,t}} g}_{\Sigma_{h,a,t-1}} \le C_t(\delta) \norm{\varphi(x)}_{\Sigma_{h,a,t-1}^{-1}}.
 \end{align*}
\end{proof}
The main challenge in the analysis will be to show that there exists an appropriate choice of $\alpha_{h,t}$ that guarantees 
that the above result holds uniformly over the value-function class $\V_{h+1,t}$ used in the definition of the confidence sets.
We note that the resulting algorithm is essentially identical to the LSVI-UCB algorithm proposed and analyzed by \citet{jin2020provably}, 
and we will accordingly refer to it by this name (that stands for ``least-squares value iteration with upper confidence bounds'').

\subsubsection{Regret Bound}
In this section we prove the regret bound of Theorem~\ref{lem:bonus_lin}, whose precise statement is as follows:
\begin{theorem}\label{thm:lin_local_reg}
With probability greater than $1-\delta$, the regret of LSVI-UCB with the choice $\lambda = 1$ and
\begin{align*}
 \alpha_{h,t} = \alpha
 &= 2 H\sqrt{d \log\pa{1 + K R^2} + \log(HA/\delta) + dA \bpa{\log(1+4HK^2R^2) + d \log(1+4R^3K^3)}} 
 \\
 &\qquad \qquad+ C_P\pa{H\sqrt{d}+1} + 1
\end{align*}
can be bounded as
\[ \Reg_T =  \wt O (A\sqrt{H^3d^3 T}).\] 
\end{theorem}
We note that the statement of the theorem is trivial when $\alpha > K$ so we will suppose that the contrary holds throughout the analysis.
The proof is a straightforward application of Theorem~\ref{thm:reggen}: given that $P\in\cP$, the regret is bounded by the sum of 
exploration bonuses, which itself can be easily bounded using Lemma~\ref{lem:sumbound}. Thus, the main challenge is to show that the 
transition model lies in the confidence set. To prove this, we observe that, thanks to the choice of exploration bonus,
the class of value functions $\V_{h+1,t}$ produced by the algorithm is composed of functions of the form
\[
 V^+_{t,h}(x) = \min\ev{H-h,\,\max\ev{\iprod{\varphi(x)}{\theta_{t,a,h}} + \alpha\norm{\varphi(x)}_{\Sigma_{t,a,h}^{-1}}}},
\] 
and the covering number of this class is relatively small.
We formalize this in the following proposition, which takes care of the probabilistic part of the analysis:
\begin{proposition}\label{prop:lin_conf}
Consider the reference model $\wh{P}_{h,a,t} = \Phi \hM_{h,a,t}$ with $\hM_{h,a,t}$ defined in Equation~\eqref{eq:Mhat}.
Then, for the choice of $\alpha$ in Theorem~\ref{thm:lin_local_reg}, the following holds simultaneously for all $x,a,h,t$, with probability at least $1-\delta$:
\[
 \sup_{\val\in\V_{h+1,t}}\iprod{P_h(\cdot|x,a) - \hP_{h,t}(\cdot|x,a)}{\val} \le \alpha \norm{\varphi(x)}_{\Sigma_{h,a,t-1}^{-1}}.
\] 
\end{proposition}
The proof of this statement is rather technical and borrows some elements of the analysis of \citet{jin2020provably}---we delegate the proof to 
Appendix~\ref{app:lin_conf}. 
Thus, we now have all the necessary ingredients to conclude the proof of Theorem~\ref{thm:lin_local_reg}. 
Indeed, since Proposition~\ref{prop:lin_conf} guarantees that the true model $P$ is always in the confidence set with probability $1-\delta$, 
and using the optimistic property of our algorithm that follows from Proposition~\ref{prop:dual_approx_1},
we can appeal to Theorem~\ref{thm:regub} to bound the regret in terms of the sum of exploration bonuses.
This in turn can be bounded by using Lemma~\ref{lem:sumbound} as follows:
\begin{align*}
 \sum_{h=1}^H\sum_{t=1}^K \textup{CB}_{h,t}^\dag(x_{h,t},a_{h,t}) &
 \le \sum_{h=1}^H \sum_a \sum_{t=1}^K \norm{\II{a_{h,t} = a}\varphi(x_{h,t})}_{\Sigma_{h,a,t-1}^{-1}} \alpha_{h,a,t}
 \\
 &\le 2 \alpha H \sqrt{d A K\log\pa{1 + K R^2/\lambda}} 
 = 2 \alpha \sqrt{H d A T\log\pa{1 + K R^2/\lambda}}.
\end{align*}
The proof is concluded by observing that $\alpha = \wt{O} (Hd\sqrt{A})$.

\subsection{Optimism in feature space through global constraints}
We now present our approach based on global confidence sets for the transition model $\tM$ that lead to an 
algorithm using exploration bonuses that can be expressed in the feature space. 
The main idea behind the algorithm is defining in each episode $t$, the confidence set $\mathcal{M}_t$ of models $\tM$ satisfying
\[
 D(\wt{M}_{h,a},\wh{M}_{h,a,t}) = \sup_{f\in\V_{h+1}} \bnorm{\bpa{\tM_{h,a} - \hM_{h,a,t}}f}_{\Sigma_{h,a,t-1}} \le \epsilon_{h,a,t}
\]
for an appropriate choice of $\epsilon_{h,a,t}$, and defining the function
\begin{align}\label{eq:featOP}
G_t(\tM) = \underset{q \in \cQ(x_1), \omega}{\text{max}} \quad & \sum_{h=1}^H \sum_a \iprod{W_{h,a,t-1} \Phi \omega_{h,a}}{r_a} &
 \\
 \text{subject to} \quad & \sum_a q_{h-1,a} = \sum_a \wt{M}_{h,a}\transpose \Phi\transpose W_{h,a,t-1} \Phi \omega_{h,a} \qquad &\forall a 
\in \cA, h=1,\dots, H        \nonumber
\\  & \Phi\transpose q_{h,a} = \Phi\transpose W_{h,a,t-1} \Phi \omega_{h,a} \qquad &\forall a \in \cA, 
h=1,\dots, H.             \nonumber
\end{align}
Clearly, if the true model $M$ is in the confidence set $\M_t$, we have $\max_{\tM \in \M_t} G_t(\tM) \ge G_t(M) = V^*_1(x_1)$. 
As phrased above, this optimization problem is intractable due to the large number of variables and constraints.
Our algorithm addresses this challenge by converting the above problem into a more tractable one that retains the optimistic property. 
In particular, our 
algorithm solves the parametric OPB equations~\eqref{eq:bellman_opt_lin} with 
confidence bonuses defined as 
\[
\CB^\dag_{h,t}(x,a) = \biprod{\varphi(x)}{B^\dag_{h,a,t}}
\]
for a vector $B^\dag_{h,a,t}\in\real^d$ chosen to maximize the following function over the convex set $\mathcal{B}_t = \bigl\{B: \norm{B_{h,a}}_{\Sigma_{h,a,t-1}} \le \epsilon_{h,a,t}\bigr\}$:
\begin{align}\label{eq:featOP}
G'_t(B) = \underset{q \in \cQ(x_1), \omega}{\text{max}} \quad & \sum_{h=1}^H \sum_a\iprod{W_{h,a,t-1} \Phi \omega_{h,a}}{r_a + \Phi B_{h,a}} &
 \\
 \text{subject to} \quad & \sum_a q_{h-1,a} = \sum_a \wh{M}_{h,a,t}\transpose \Phi\transpose W_{h,a,t-1} \Phi \omega_{h,a} \qquad &\forall a 
\in \cA, h=1,\dots, H        \nonumber
\\  & \Phi\transpose q_{h,a} = \Phi\transpose W_{h,a,t-1} \Phi \omega_{h,a} \qquad &\forall a \in \cA, 
h=1,\dots, H.             \nonumber
\end{align}
This definition is easily seen to be equivalent to the one given in the statement of Theorem~\ref{prop:regfeat} through 
basic LP duality (cf.~Section~\ref{sec:background}). 
Our analysis will take advantage of the fact that our exploration bonuses are linear in the feature representation, which eventually
yields value functions of the following form:
\begin{equation}\label{eq:vform}
 V^\dag_{h,t}(x) = \min\ev{H-h+1,\,\max_a\biprod{\varphi(x)}{\theta^\dag_{h,a,t}}},
\end{equation}
for some $\theta^\dag_{h,a,t}\in\real^d$, which implies that the class of functions $\V_{h+1,t}$ is simpler than in the case LSVI-UCB. 
The algorithm is justified by the following property:
\begin{proposition}\label{prop:Gbounds}
For any episode $t$, let the functions $G_t$ and $G'_t$ be defined as above and let $\mathcal{B}_t = \bigl\{B: \norm{B_{h,a}}_{\Sigma_{h,a,t-1}} \le \epsilon_{h,a,t}\bigr\}$.
Then, $\max_{B\in\mathcal{B}_t} G'_t(B) \ge \max_{\tM \in \M_t} G_t(\tM)$.
\end{proposition}
\begin{proof}
Let us fix a model $\tM \in \M_t$, introduce the notation $Z_{h,a,t} = \bpa{\tM_{h,a} -\hM_{h,a,t}}V_{h+1,t}$, and notice that 
$Z_t \in \mathcal{B}_t$ due to the definition of $\M_t$. 
The proof relies on expressing the values of $G_t(\tM)$ and $G'_t(B)$ through the OPB equations~\eqref{eq:bellman_opt_lin} defining them. 
Indeed, for a fixed $\tM$, the value of $G_t(\tM)$ can be expressed through standard LP duality as exposed in Section~\ref{sec:background}. 
To express $G_t(\tM)$, let $U_t$ stand for the value function defined through the system of equations
\begin{equation*}
\begin{split}
 \theta_{h,a,t} &= \rho_a + \tM_{h,a} U_{h+1,t} = \rho_a + \bpa{\tM_{h,a} -\hM_{h,a,t}}U_{h+1,t} + \hM_{h,a,t} U_{h+1,t} 
 	\\ &= \rho_a + Z_{h,a,t} + \hM_{h,a,t} U_{h+1,t},
 \\
 U_{h+1,t}(x) &= \max_{a} \iprod{\varphi(x)}{\theta_{h+1,a,t}}
\end{split}
\end{equation*}
that have to be satisfied for all $x,a,h$. Then, it is easy to see that $G_t(\tM) = U_{1,t}(x_1)$. Notice that this can be understood as the solution of 
the OPB equations~\eqref{eq:bellman_opt_lin} with exploration bonus $\CB_{h,t}(x,a) = \iprod{\varphi(x)}{Z_{h,a,t}}$. On the other hand, $G'_t(B)$ can be 
expressed as $U'_{1,t}(x_1)$ with $U'_{t}$ is defined through the system of equations
\begin{equation*}
\begin{split}
 \theta'_{h,a,t} &= \rho_a + \tM_{h,a,t} U'_{h+1,t} 
 \\
 U'_{h,t}(x) &= \max_{a} \iprod{\varphi(x)}{\theta'_{h,a,t} + B_{h,a,t}}.
\end{split}
\end{equation*}
It is then easy to verify that $G_t(\tM) = G'_t(Z)$ and, using $Z\in\mathcal{B}_t$, that $G'_t(Z) \le \max_{B\in\mathcal{B}_t} G'_t(B)$. This concludes the proof since the inequality must hold for any model $\tM \in \M_t$.
\end{proof}
Notably, the above proposition ensures that the value function $V^\dag_t$ arising from the OPB equations~\eqref{eq:bellman_opt_lin} with bonus $\CB^\dag_{h,t}(x,a) = \biprod{\varphi(x)}{B_{h,a,t}^\dag}$ is optimistic in the sense that $V^\dag_{1,t}(x_{1,t}) \ge G_t(\tM) \ge V^*_1(x_{1,t})$. This enables us to apply the general regret bound of Theorem~\ref{thm:regub} to establish a performance guarantee for the resulting algorithm. We provide this analysis in the next section.

From the above formulation, it is readily apparent that, since $G'$ is a maximum of linear functions, it is a convex function of $B$, and thus maximizing it over a 
convex set is potentially still very challenging. We note that this optimization problem is essentially identical to the one faced 
by the seminal LinUCB algorithm for linear bandits~\citep{DHK08,abbasi2011improved}, which is known to be computationally intractable for general decision sets.
This is to be contrasted with the algorithms described in previous parts of this paper, which are efficiently implementable through dynamic programming. %
Indeed, despite being of a similar form, the simplicity of these previous methods stem from the local nature of their confidence sets which was seen to lead 
to exploration bonuses that can be set independently for each state and computed via dynamic programming. This is no longer possible for the exploration bonuses used in this section, which are set through a global parameter vector $B$. Intuitively, this prevents the application of dynamic-programming methodology which heavily relies on the ability of breaking down an optimization problem into a set of local optimization problems (often referred to as the ``principle of optimality'' in this context \citep{Ber07:DPbookVol2}). It remains an open problem to find an efficient implementation of this method.

It is interesting to note that our algorithm essentially coincides with the \eleanor method proposed 
very recently by \citet{zanette2020learning}, up to minor differences. Their analysis is more 
general than ours as they considered the significantly harder case of learning with misspecified 
linear models that our analysis doesn't account for. Nevertheless, our analysis is substantially simplified
by our model-based perspective that sheds new light on the algorithm. In particular, while 
\citet{zanette2020learning} do not provide a substantial discussion of the computational challenges 
associated with \eleanor, our formulation clearly highlights the convexity of the objective 
function optimized by the algorithm and the relation with LinUCB. We believe that our model-based 
perspective can provide further insights into this challenging problem in the future, and particularly
that it will remain useful when analyzing misspecified linear models.

\subsubsection{Regret bound}
We now prove our main result regarding the algorithm: the regret bound claimed in Theorem~\ref{prop:regfeat}. In particular, 
the detailed statement of this result is as follows:
\begin{theorem}\label{thm:lin_global_reg}
 With probability greater than $1-\delta$, the regret of our algorithm with  $\lambda = 1$ 
 for 
\begin{align*}
   \epsilon_{h,a,t} 
= \epsilon =& 2 H\sqrt{d \log\pa{1 + K R^2} + dA \log(1+4K^2HR^3) +  \log\pa{\redd{HA}/ \delta}} 
 \\
 &+ \lambda^{1/2} \pa{C_P \sqrt{d} + 1 + C_P}
\end{align*}
satisfies
\[ 
\Reg_T = \wt O(dA\sqrt{H^3T}).
\]
\end{theorem}
The key idea of the analysis is to use Proposition~\ref{prop:Gbounds} to establish the optimistic property of the 
algorithm and use Theorem~\ref{thm:regub} to bound the regret by the sum of exploration bonuses. The only remaining 
challenge is to prove that, with high probability, the true model lies in the confidence sets specified in Equation~\eqref{eq:global_const}.
The following proposition guarantees that this is indeed true:
\begin{proposition}\label{prop:lin_conf_global}
Consider the reference model $\wh{P}_{h,a,t} = \Phi \hM_{h,a,t}$ with $\hM_{h,a,t}$ defined in Equation~\eqref{eq:Mhat}.
Then, for the choice of $\epsilon$ in Theorem~\ref{thm:lin_global_reg}, 
the following holds simultaneously for all $a,h,t$, with probability at least $1-\delta$:
\[
 \sup_{f\in\V_{h+1,t}} \bnorm{\bpa{M_{h,a} - \hM_{h,a,t}}f}_{\Sigma_{h,a,t-1}} \le \epsilon_{h,a,t}.
\] 
\end{proposition}
The proof relies on a covering argument similar to the one we used for proving Proposition~\ref{prop:lin_conf}, exploiting 
the fact that the value function class $\V_{h+1,t}$ is composed of slightly simpler functions. The proof is deferred to 
Appendix~\ref{app:lin_conf_global}.
Thus, we can conclude the proof of Theorem~\ref{thm:lin_global_reg} as follows. Taking advantage of the fact that the algorithm 
follows the optimal policy corresponding to the solution of the OPB~equations~\eqref{eq:bellman_opt_lin}, we can use 
the general guarantee of Theorem~\ref{thm:regub} and bound the regret of the algorithm as the sum of the exploration bonuses. 
Noticing that the bonuses can be upper-bounded as
\[
 \textup{CB}_{h,t}^\dag(x,a) = \biprod{\varphi(x)}{B^\dag_{h,a,t}} \le 
 \norm{\varphi(x_{h,t})}_{\Sigma_{h,a,t-1}^{-1}} \bnorm{B^\dag_{h,a,t}}_{\Sigma_{h,a,t-1}} \le \norm{\varphi(x_{h,t})}_{\Sigma_{h,a,t-1}^{-1}} \epsilon_{h,a,t},
\]
where the last step follows from the fact that $B^\dag_{h,a,t} \in \mathcal{B}_t$, the sum of confidence bonuses can be bounded by appealing to 
Lemma~\ref{lem:sumbound}:
\begin{align*}
 \sum_{h=1}^H\sum_{t=1}^K \textup{CB}_{h,t}^\dag(x_{h,t},a_{h,t}) &
 \le \sum_{h=1}^H \sum_a \sum_{t=1}^K \norm{\II{a_{h,t} = a}\varphi(x_{h,t})}_{\Sigma_{h,a,t-1}^{-1}} \epsilon_{h,a,t}
 \\
 &\le 2 \epsilon H \sqrt{d A K \log\pa{\frac{1 + K R^2/\lambda}{\delta}}} 
 = 2 \epsilon \sqrt{H d A T \log\pa{\frac{1 + K R^2/\lambda}{\delta}}}.
\end{align*}
Setting $\lambda = 1$ and noticing that $\epsilon = \wt O(H\sqrt{dA})$ concludes the proof of Theorem~\ref{thm:lin_global_reg}.

\subsection{Technical proofs}\label{app:technical}
\subsubsection{Proof of Proposition~\ref{prop:dual_approx_1}}\label{app:dual_approx_1}
We first note that, since $\wh{P}_{h,a} = \Phi \wh{M}_{h,a}$ and using the second constraint, the first 
constraint in the optimization problem can be rewritten as
\[
 \sum_a {q}_{h+1,a} =\sum_a \wh{M}_{h,a} \Phi W_{h,a}\Phi\omega_{h,a} + \sum_a \pa{\wt{P}_{h,a} - \wh{P}_{h,a}} 
{q}_{h,a}.
\]

Using this, we use a similar argument to Lemma~\ref{lem:strongcomp} to show that strong duality and the KKT conditions hold. We reparameterize by defining $J_h(x,a,x') = q_h(x,a) \wt P_h(x'|x,a)$ and observe that the last constraint in \eqref{eq:linopt1} is can be written as $D(J_h(x,a, \cdot) , \wh P_h(\cdot|x,a) \sum_{x'}J_h(x,a,x')) \leq \epsilon_h(x,a) \sum_{x'} J_h(x,a,x')$ which is convex in $J$. It can also be easily observed that the first two constraints, and the objective are linear in $q,J,\omega$. Thus strong duality holds, and the optimal value of the reparameterized optimization problem is equal to the optimal value of the corresponding Lagrangian dual problem. As in the proof of Lemma~\ref{lem:strongcomp}, by using the reverse reparameterization, we can see that the value of the Lagrangian of the modified problem is equal to that of the original problem in \eqref{eq:linopt1}.  Hence, strong duality holds for \eqref{eq:linopt1}. It then follows that the KKT conditions also hold for this problem.

Given strong duality, we can find the dual of the problem in \eqref{eq:linopt1} by considering the Lagrangian.
The partial Lagrangian of the optimization problem without the last constraint of the primal can be 
written as
\begin{align}
 \cL(q,\kappa,\omega; \val, \theta) 
 =& \sum_{h,a} \iprod{W_{h,a}\Phi\omega_{h,a}}{r_a + \wh{P}_{h,a} 
\val_{h+1} - \Phi \theta_{h,a}} \nonumber
\\
&+\sum_{x,a,h} q_h(x,a) \bigg(\pa{\Phi \theta_{h,a}}(x) + \sum_y \kappa_h(x,a,y) \val_{h+1}(y)  - 
\val_h(x) \bigg) + \val_1(x_1), \label{eq:lag_lin}
\end{align}
for $\kappa_h(x,a,y) =\wt{P}_h(y|x,a)-\wh{P}_h(y|x,a)$. 
Then, by strong duality, the optimal value of the primal is equal to 
\[ \min_{\val,\theta}  \max_{\substack{q \geq 0, \wt P \in \cP \\ \omega, \kappa}} \cL(q,\kappa,\omega; \val, \theta).  \]
Observing that $q_h(x,a) \geq 0$ and using the definition of $\kappa_h(x,a,\cdot)$, we can consider the inner maximization over $\wt P_h(\cdot|x,a) \in \cP_h(x,a)$. We get,

\begin{align*}
\max_{ \wt{P}_h(\cdot|x,a) \in \cP_h(x,a)} \sum_y(\wt P_h(y|x,a) - \wh P_h(y|x,a)) \val_{h+1}(y)  = D_*(\val_{h+1} | \wh{P}, \epsilon)
\end{align*}
by definition of the conjugate.
Substituting this back into \eqref{eq:lag_lin}, we can find the dual from this Lagrangian by a similar technique to Proposition~\ref{prop:dual}. In particular, observe that the objective function will be given by $\val_1(x_1)$. To define the constraints, note that if $\max_{\omega} \sum_{h,a} \iprod{W_{h,a}\Phi\omega_{h,a}}{r_a + \wh{P}_{h,a} \val_{h+1} - \Phi \theta_{h,a}} <\infty$, it must be the case that $\iprod{W_{h,a}\Phi}{r_a + \wh{P}_{h,a} \val_{h+1} - \Phi \theta_{h,a}} =0$, and likewise if $\max_{q>0} \sum_{x,a,h} q_h(x,a) (\pa{\Phi \theta_{h,a}}(x) +D_*(\val_{h+1}|\epsilon_{h,a}, \wh P_h(\cdot|x,a))  - \val_h(x)) <\infty$, it must be the case that $\pa{\Phi \theta_{h,a}}(x) +D_*(\val_{h+1}|\epsilon_{h,a}, \wh P_h(\cdot|x,a))  - \val_h(x)\leq 0$.

Thus the dual optimization problem can be written,
\begin{align*}
\underset{\val}{\text{minimize }}  \;& \val_1(x_1) \label{eq:dualtab}
 \\ 
  \text{Subject to } \;& \val_h(x) \geq \pa{\Phi \theta_{h,a}}(x) + D_*\pa{\val_{h+1} \middle| \epsilon_h(x,a),\wh{P}_h(\cdot|x,a)} \qquad &\forall  (x,a) \in \cZ, h \in [H] \nonumber
   \\& \pa{\Phi W_{h,a} \Phi} \theta_{h,a} =   \Phi\transpose W_{h,a}\pa{r_a + \wh{P}_{h,a}\val_{h+1}} \qquad &\forall a \in \cA, h \in [H].  \nonumber
\end{align*}

It is easily seen that the solution to this can be found by solving the optimistic parametric Bellman equations in \eqref{eq:bellman_opt_lin} with $\CB_{h,t}(x,a) =  D_*\pa{\val_{h+1,t} \middle| \epsilon_{h,t}(x,a),\wh{P}_{h,t}(\cdot|x,a)}$  via backwards recursion.\qed

\subsubsection{The proof of Proposition~\ref{prop:lin_conf}}\label{app:lin_conf}
The proof follows from a construction proposed by \citet{jin2020provably}: it relies on taking a union bound over an appropriately chosen 
covering  of the class of value functions in stage $h+1$ that can be ever produced by solving the optimistic Bellman equations~\eqref{eq:bellman_opt_lin}.
For this purpose, we need the following technical result 
that bounds the covering number of this set:
\begin{lemma} \label{lem:cover1}
Let $\cN(\V,\varepsilon)$ be the $\varepsilon$-covering number of the set $\V$ with respect to the distance $\|V-V'\|_\infty = \sup_{x \in \cS} |V(x) - V'(x)|$. Then, for any stage $h=1,\dots, H$ and episode $t$,
\[ \log(\cN(\V_{h+1,t}, \varepsilon)) \leq Ad \log(1+4tHR/(\lambda \varepsilon)) + d^2A \log(1+4R\alpha/(\lambda \varepsilon^2)) \]
where $R$ is such that $\|\varphi(x)\|_2 \leq R \, \forall x \in \cS$, $\lambda$ is such that the minimum eigenvalue, $\lambda_{\min}(\Sigma_{h,a,t})\geq \lambda \,\, \forall a \in \cA, h\in [H], t \in [K]$.
\end{lemma}
The proof of Lemma~\ref{lem:cover1} is similar to that of Lemma D.6 of \cite{jin2020provably}, and exploits that the class $\V_{h+1,t}$ is 
parametrized smoothly by $\theta$ and $\Sigma$. We relegate the proof to Appendix~\ref{app:cover1}.
As for the proof of Proposition~\ref{prop:lin_conf}, let us fix any $h,a$, $\varepsilon>0$ and any $\val\in\V_{h+1,t}$, and let $\tV$ be in the $\varepsilon$-covering of $\V_{h+1,t}$ defined in Lemma~\ref{lem:cover1} such that
$\|\val - \tV\|_\infty \leq \varepsilon$. Then, we have
\begin{align*}
 &\iprod{P_h(\cdot|x,a) - \hP_{h,t}(\cdot|x,a)}{\val} 
 \\
 &\qquad\qquad= \iprod{P_h(\cdot|x,a) - \hP_{h,t}(\cdot|x,a)}{\tV} + \iprod{P_h(\cdot|x,a) - \hP_{h,t}(\cdot|x,a)}{\val - \tV}.
 \end{align*}
 The second term can be bounded by introducing the notation $\tg = \val- \tV$ and writing
\begin{align*}
\iprod{P_h(\cdot|x,a) - \hP_{h,t}(\cdot|x,a)}{\tg} &= \iprod{\varphi(x)}{\pa{M_h - \hM_{h,a,t}}\tg} 
\le \norm{\varphi(x)}_{\Sigma_{h,a,t}^{-1}} \norm{\pa{M_h - \hM_{h,a,t}}\tg}_{\Sigma_{h,a,t}}
\\
&\le \varepsilon \pa{\lambda^{-1/2} tR + \lambda^{1/2}  C_P}\norm{\varphi(x)}_{\Sigma_{h,a,t}^{-1}},
\end{align*}
where we used Lemma~\ref{lem:mhatscale} with $B= \varepsilon$ in the last step. 
As for the first term, we use a union bound over all $\tV$ in the $\varepsilon$-covering of $\V_{h+1,t}$ and 
Lemma~\ref{lem:conc_lin_local}. Denoting the covering number as $\cN_\varepsilon$ and setting $\delta' = \delta/HA$, we can see that for any $\wt V$ in the $\varepsilon$-covering, with probability greater than $1-\delta'$, we have
\begin{align*}
&\iprod{P_h(\cdot|x,a) - \hP_{h,t}(\cdot|x,a)} {\wt V} \leq  \|\varphi(x)\|_{\Sigma_{h,a,t}^{-1}} C_t(\delta'/\cN_\varepsilon),
\end{align*}
which can be further bounded as
\begin{align*}
 &C_t(\delta'/\cN_\varepsilon) - C_P H \sqrt{\lambda d}= 2 H\sqrt{d \log\pa{1 + t R^2/\lambda} + \log\pa{\mathcal N_\varepsilon / \delta'}} 
 \\
 &\ \le 2 H\sqrt{d \log\pa{1 + t R^2/\lambda} + \log(1/\delta') + dA \log(1+4HtR/(\lambda \varepsilon)) + d^2A \log(1+4R\alpha/(\lambda \varepsilon^2))}
 \\
 &\ \le 2 H\sqrt{d \log\pa{1 + t R^2/\lambda} + \log(1/\delta') + dA \log(1+4Ht^2R^2/\lambda^2 ) + d^2A \log(1+4R^3Kt^2/\lambda^3 )},
\end{align*}
where we set $\varepsilon = \lambda/(tR)$ and used the condition $\alpha \le K$ in the last step. With the same choice of $\varepsilon$,
we also have
\[
 \varepsilon \pa{\lambda^{-1/2} tR + \lambda^{1/2}  C_P} \le \lambda^{1/2} \pa{1 + C_P}.
\]
Noticing that the sum of the two latter terms is bounded by $\alpha$ and taking a union bound over all $h,a$ concludes the proof.
\qed

\subsubsection{The proof of Lemma~\ref{lem:cover1}}\label{app:cover1}
We first note that, due to the definition of the parameter vectors $\theta_{h,a,t}^+$ as the solution of the OPB~equations~\eqref{eq:bellman_opt_lin} with
$\bnorm{V_{h+1,t}^+}_\infty \le H$, we have 
\[
\bnorm{\theta_{h,a,t}^+} \le \frac{tHR}{\lambda} \stackrel{\text{def}}{=} \beta,
\]
where the inequality follows from Lemma~\ref{lem:mhatscale}. To preserve clarity of writing, we omit explicit references to $t$ below.
By design of the algorithm, we can see that the value functions can be written 
with the help of the function $U_{h,\theta,\Sigma}$ defined as
\[
 U_{h,\theta,\Sigma}(x) = \min\ev{ H-h+1, \max_{a \in \cA}\ev{ \iprod{\varphi(x)}{\theta_{h,a}} + \alpha \|\varphi(x)\|_{\Sigma^{-1}_{h,a}} }}
\]
for some $\alpha>0$.
Indeed, the class of value functions can be written as
\[
 \V_h = \ev{ U_{h,\theta,\Sigma} : \ \ \max_a \norm{\theta_{h,a}} \le \beta,\ \  \max_a \opnorm{\Sigma_{h,a}^{-1}}\le 1/\lambda}.
\]

We show below that $U_{h,\theta,\Sigma}$ is a smooth function of the parameters $\theta_{h,a}$ and $\Sigma_{h,a}^{-1}$, which will allow us to prove a tight bound on the covering number of the class $\V_{h}$.
Indeed, letting $V_h = U_{h,\theta,\Sigma}$ and $\tV_h = U_{h,\ttheta,\wt\Sigma}$ for an arbitrary set of parameters $\theta,\Sigma,\ttheta,\wt\Sigma$, we have
\begin{align*}
\|V_h - \wt V_h\|_\infty &= \sup_{x \in \cS} \bigg|  \min\{ H-h+1, \max_{a \in \cA}\{ \varphi(x)\transpose \theta_{h,a} + \alpha \| \varphi(x)\|_{\Sigma^{-1}_{h,a}} \} \}  
	\\ & \hspace{100pt} -  \min\{ H-h+1, \max_{a \in \cA}\{ \varphi(x)\transpose \wt \theta_{h,a} + \alpha \| \varphi(x)\|_{\wt \Sigma^{-1}_{h,a}} \} \}  \bigg |
\\ & \leq \sup_{x \in \cS} \bigg|\max_{a \in \cA}\{ \varphi(x)\transpose \theta_{h,a} + \alpha \| \varphi(x)\|_{\Sigma^{-1}_{h,a}} \}  - \max_{a \in \cA}\{ \varphi(x)\transpose \wt \theta_{h,a} + \alpha \| \varphi(x)\|_{\wt \Sigma^{-1}_{h,a}} \} \bigg |
\\ & \leq \sup_{x \in \cS, a \in \cA} \bigg| \varphi(x)\transpose \theta_{h,a} + \alpha \| \varphi(x)\|_{\Sigma^{-1}_{h,a}}  - \varphi(x)\transpose \wt \theta_{h,a} + \alpha \| \varphi(x)\|_{\wt \Sigma^{-1}_{h,a}}  \bigg| 
\\ & \leq \sup_{x \in \cS,a \in \cA} \bigg| \varphi(x)\transpose (\theta_{h,a} - \wt \theta_{h,a})  + \sqrt{ \varphi(x) \transpose (\alpha \Sigma_{h,a}^{-1} - \alpha \wt \Sigma_{h,a}^{-1}) \varphi(x)} \bigg|
\\ & \leq \sup_{a \in \cA} R\| \theta_{h,a} - \wt \theta_{h,a}\|_2 + \sup_{a \in \cA} R\|\alpha \Sigma_{h,a}^{-1} - \alpha \wt \Sigma_{h,a}^{-1} \|_{\text{op}}
\\ & \leq  \sup_{a \in \cA} R \| \theta_{h,a} - \wt \theta_{h,a}\|_2 + \sup_{a \in \cA} R \alpha\| \Sigma_{h,a}^{-1} - \wt \Sigma_{h,a}^{-1} \|_F
\end{align*}
since $\|\varphi(x)\|_2 \leq  R$  and we have used $\|A\|_{\text{op}}$ to denote the operator norm and $\|A\|_F$ the Frobenius norm of a matrix $A$.

We then note that the $\varepsilon/2$-covering number of the set $\Theta = \{ (\theta_a)_{a \in \cA}: \theta_a \in \bR^d, \, \sup_{a \in \cA}\|\theta_a\|_2 \leq \beta \}$ is bounded by $(1+4\beta/\varepsilon)^{Ad}$, and that $\varepsilon/2$-covering number of the set $\Gamma = \{  (\Sigma_a)_{a \in \cA}: \Sigma_a \in \bR^{d\times d}, \,\sup_{a \in \cA} \|\Sigma_a\|_F \leq 1/\lambda \}$ is bounded by $(1+4/(\lambda \varepsilon^2))^{d^2A}$. %
These results follow due to the standard fact that the $\varepsilon$-covering number of a ball in $\bR^d$ with radius $R>0$ with $\ell_2$ distance is bounded by $(1+2R/\varepsilon)^d$, 
and that $\Theta$ and $\Gamma$ are $(dA)$-dimensional and $(d^2A)$-dimensional, respectively.

From the above discussion, we can conclude that for any $V_h \in \V_h$, there is a $\wt V_h$ parameterized by $\wt \theta_h$ in the $\varepsilon/2$-covering of $\Theta_h$, and $\wt \Sigma_h$ in the $\varepsilon/2$-covering of $\Gamma_h$ such that,
\[ \|V_h - \wt V_h\|_\infty \leq R \varepsilon/2 + R\alpha \varepsilon/2. \]
By rescaling of the covering numbers, we can see that the logarithm of the $\varepsilon$-covering number of $\V_h$ can be bounded by
\begin{align*}
\log(\cN(\V_h,\varepsilon)) &\leq \log(\cN(\Theta_h, \varepsilon/(2R))) + \log(\cN(\Gamma_h, \varepsilon/(2\alpha R)) 
\\ &\leq Ad \log(1+4\beta R/\varepsilon) + d^2A \log(1+4R\alpha/(\lambda \varepsilon^2)).
\end{align*}
Substituting in $\beta=\frac{tHR}{\lambda}$ gives the result.

\qed

\subsubsection{The proof of Proposition~\ref{prop:lin_conf_global}}\label{app:lin_conf_global}
The proof is similar to that of Proposition~\ref{prop:lin_conf}, in that it also relies on a covering argument to prove uniform 
convergence over the set of potential value functions.
The following technical result bounds the covering number of this set:
\begin{lemma} \label{lem:cover2}
Let $\cN(\V,\varepsilon)$ be the $\varepsilon$-covering number of some set $\V$ with respect to the distance $\|V-V'\|_\infty = \sup_{x \in \cS} |V(x) - V'(x)|$. Then, for any stage $h=1,\dots, H$ and episode $t$,
\[ \log(\cN(\V_{h+1,t}, \varepsilon)) \leq dA \log(1+4tHR^2/(\varepsilon\lambda)).\]
\end{lemma}
To reduce clutter, we defer the proof to Appendix~\ref{app:cover2}. To proceed, we fix $h,a$, $\varepsilon > 0$ and an arbitrary 
$V\in\V_{h+1,t}$, and consider a $\tV$ in the covering defined above such that $\bnorm{V-\tV}_\infty\le \varepsilon$. Then, by the triangle inequality, 
we have
\begin{align*}
 \bnorm{\bpa{M_{h,a} - \hM_{h,a,t}}V}_{\Sigma_{h,a,t-1}} &\le \bnorm{\bpa{M_{h,a} - \hM_{h,a,t}}\tV}_{\Sigma_{h,a,t-1}} + \bnorm{\bpa{M_{h,a} - \hM_{h,a,t}}\bpa{V - \tV}}_{\Sigma_{h,a,t-1}}
 \\
 &\le \bnorm{\bpa{M_{h,a} - \hM_{h,a,t}}\tV}_{\Sigma_{h,a,t-1}} + \varepsilon \pa{\lambda^{-1/2} tR + \lambda^{1/2}  C_P},
\end{align*}
where we used Lemma~\ref{lem:mhatscale} with $B= \varepsilon$ in the last step. 
Setting $\delta' = \delta / (HA)$, the first term can be bounded with probability at least $1-\delta'$ 
by exploiting that $\tV$ is in the covering, and using the union bound to show that for every such $\tV$, we simultaneously have
\begin{align*}
 &\bnorm{\bpa{M_{h,a} - \hM_{h,a,t}}\tV}_{\Sigma_{h,a,t-1}} \le C_t(\delta'/\mathcal N_\varepsilon) = 2 H\sqrt{d \log\pa{1 + t R^2/\lambda} + \log\pa{\mathcal N_\varepsilon / \delta'}} + C_P H \sqrt{\lambda d}
 \\
 &\qquad\le  2 H\sqrt{d \log\pa{1 + t R^2/\lambda} + dA \log(1+4tHR^2/(\varepsilon\lambda)) +  \log\pa{ 1/ \delta'}} + C_P H \sqrt{\lambda d}
\end{align*}
Putting the two bounds together and setting $\varepsilon = \lambda/(tR)$ gives
\begin{align*}
 \bnorm{\bpa{M_{h,a} - \hM_{h,a,t}}V}_{\Sigma_{h,a, t-1}} \le& 2 H\sqrt{d \log\pa{1 + t R^2/\lambda} + dA \log(1+4t^2HR^3/\lambda^2) +  \log\pa{HA/ \delta}} 
 \\
 &+ \lambda^{1/2} \pa{C_P H \sqrt{d} + 1 + C_P}.
\end{align*}
This is clearly upper-bounded by the chosen value of $\epsilon$. Taking a union bound over all $h,a$ concludes the proof. \qed

\subsubsection{The proof of Lemma~\ref{lem:cover2}}\label{app:cover2}
The proof is similar to that of Lemma~\ref{lem:cover1}, although simpler due to the simpler form of the value functions in this case.
We stary by noting that, due to the definition of the parameter vectors $\theta_{h,a,t}^+$ as the solution of the OPB~equations~\eqref{eq:bellman_opt_lin} with
$\bnorm{V_{h+1,t}^+}_\infty \le H$, we have 
\[
\bnorm{\theta_{h,a,t}^+} \le \frac{tHR}{\lambda} \stackrel{\text{def}}{=} \beta,
\]
where the inequality follows from Lemma~\ref{lem:mhatscale}. 
Given the definition of the algorithm, it is easy to see that the value functions can be with the help of the function $U_{h,\theta}$ defined as
\[
 U_{h,\theta}(x) = \min\ev{ H-h+1, \max_{a \in \cA} \iprod{\varphi(x)}{\theta_{h,a}}},
\]
in the form $V_{h,t} = U_{h,\theta}$ for some $\theta$ with norm bounded by $\beta$. Thus, the set of value functions can be written as
\begin{align*}
 \V_h = \ev{ U_{h,\theta} :\ \ \norm{\theta} \le \beta }.
 \end{align*}
We show below that $U$ is a smooth function of $\theta$, which will allow us to prove a tight bound on the covering number of the class $\V_{h}$. Indeed, this can be seen by
\begin{align*}
 \infnorm{U_{h,\theta} - U_{h,\theta'}} &\le \sup_{x\in\Sw} \left| \max_{a \in \cA} \iprod{\varphi(x)}{\theta_{h,a}} - \max_{a \in \cA} \iprod{\varphi(x)}{\theta_{h,a}'}\right|
 \le \sup_{x\in\Sw} \max_{a\in\cA} \left|\iprod{\varphi(x)}{\theta_{h,a} - \theta_{h,a}'}\right| 
 \\
 &\le R \max_a \norm{\theta_{h,a} - \theta_{h,a}'}.
\end{align*}
Thus, the $\varepsilon/2$-covering number of the set $\Theta = \{ (\theta_a)_{a \in \cA}: \theta_a \in \bR^d, \, \sup_{a \in \cA}\|\theta_a\|_2 \leq \beta \}$ is bounded by $(1+4\beta/\varepsilon)^{Ad}$, which follows from the standard fact that the $\varepsilon$-covering number of a ball in $\bR^d$ with 
radius $c>0$ in terms of the $\ell_2$ distance is bounded by $(1+2c/\varepsilon)^d$. Thus, we have that for any $V_h \in \V_h$, 
there exists a $\wt V_h$ parameterized by $\wt \theta_h$ in the $\varepsilon/2$-covering of $\Theta_h$ such that,
\[ \|V_h - \wt V_h\|_\infty \leq R \varepsilon/2. \]
By rescaling of the covering numbers, we can see that the logarithm of the $\varepsilon$-covering number of $\V_h$ can be bounded by
\begin{align*}
\log(\cN(\V_h,\varepsilon)) &\leq \log(\cN(\Theta_h, \varepsilon/(2R))) \leq dA \log(1+4\beta R/\varepsilon),
\end{align*}
giving the result.

\end{document}